\documentclass{article}

    \usepackage[preprint]{neurips_2024}

\usepackage{makecell}
\usepackage{multirow}
\usepackage[utf8]{inputenc} 
\usepackage[T1]{fontenc}    
\usepackage{hyperref}       
\usepackage{url}            
\usepackage{booktabs}       
\usepackage{amsfonts}       
\usepackage{nicefrac}       
\usepackage{microtype}      
\usepackage{xcolor}         
\usepackage{xspace}
\usepackage{graphicx}
\usepackage{enumitem}
\usepackage{amsmath}
\usepackage{mathtools}
\usepackage[ruled,vlined,linesnumbered]{algorithm2e}
\usepackage{wrapfig}
\usepackage{algorithmic}
\hypersetup{colorlinks,citecolor={blue}}
\usepackage[english]{babel}
\usepackage{amsthm}
\newtheorem{theorem}{Theorem}
\newtheorem{lemma}{Lemma}

\newtheorem{assumption}{Assumption}
\usepackage{subfigure}
\usepackage{amssymb}
\newcommand{\method}{PMformer\xspace}
\newcommand{\ensemble}{an inference technique\xspace}

\title{Partial-Multivariate Model for Forecasting}

\author{%
Jaehoon Lee$^1$ \ \ \ \  Hankook Lee$^{1,2}$ \ \ \ \ Sungik Choi$^1$ \ \ \ \ Sungjun Cho$^1$ \ \ \ \ Moontae Lee$^{1,3}$ \\
$^1$LG AI Research \ \ \ \ $^2$Sungkyunkwan University \ \ \ \ $^3$University of Illinois Chicago
}

\begin{document}

\maketitle

\begin{abstract}


When solving forecasting problems including multiple time-series features, existing approaches often fall into two extreme categories, depending on whether to utilize inter-feature information:  univariate and complete-multivariate models. Unlike univariate cases which ignore the information, complete-multivariate models compute relationships among a complete set of features.
However, despite the potential advantage of leveraging the additional information, complete-multivariate models sometimes underperform univariate ones. 
Therefore, our research aims to explore a middle ground between these two by introducing what we term \emph{Partial-Multivariate} models where a neural network captures only partial relationships, that is, dependencies within subsets of all features. To this end, we propose \method, a Transformer-based partial-multivariate model, with its training algorithm. We demonstrate that \method outperforms various univariate and complete-multivariate models, providing a theoretical rationale and empirical analysis for its superiority. Additionally, by proposing \ensemble for \method, the forecasting accuracy is further enhanced. Finally, we highlight other advantages of \method: efficiency and robustness under missing features.

\end{abstract}

\section{Introduction}

\emph{Time-series forecasting} is a fundamental machine learning task that aims to predict future events based on past observations, requiring to capture temporal dynamics. A forecasting problem often includes interrelated multiple variables (\textit{e.g.}, multiple market values in stock price forecasting). For decades, the forecasting task with multiple time-series features has been of great importance in various applications such as health care~\citep{nguyen2021forecasting, jones2009multivariate}, meteorology~\citep{sanhudo2021multivariate,angryk2020multivariate}, and finance~\citep{article1,Mehtab_2021}. 

For this problem, there have been developed a number of deep-learning models,
 including linear models~\citep{tsmixer,dlinear}, state-space models~\citep{liang2024bimamba4ts,statespace2}, recurrent neural networks (RNNs)~\citep{rnn1,rnn2}, convolution neural networks (CNNs)~\citep{cnn2,cnn1}, and Transformers~\citep{informer,pyraformer}. These models are typically categorized by the existence of modules that capture inter-feature information, falling into two extremes: \textit{(i)} univariate and \textit{(ii)} complete-multivariate models\footnote{To differentiate from existing multivariate models that capture dependencies among a complete set of features, we refer to them as complete-multivariate models, while our new approaches, which capture partial dependencies by sampling subsets of features, are termed partial-multivariate models.}. While univariate models only capture temporal dependencies, complete-multivariate ones incorporate additional modules that account for complete dependencies among all given features.

However, although complete-multivariate models have potential advantages over univariate ones by additionally utilizing inter-feature information, some studies have demonstrated that complete-multivariate models are sometimes inferior to univariate ones in time-series forecasting.~\citep{dlinear,onenet,fits} For example, in Table~\ref{tbl:main}, PatchTST (a univariate Transformer-based model)~\citep{patchtst} outperforms Crossformer (a complete-multivariate Transformer-based model)~\citep{crossformer} by a large margin.

\begin{figure*}[t]
    \centering
    \includegraphics[width=\textwidth]{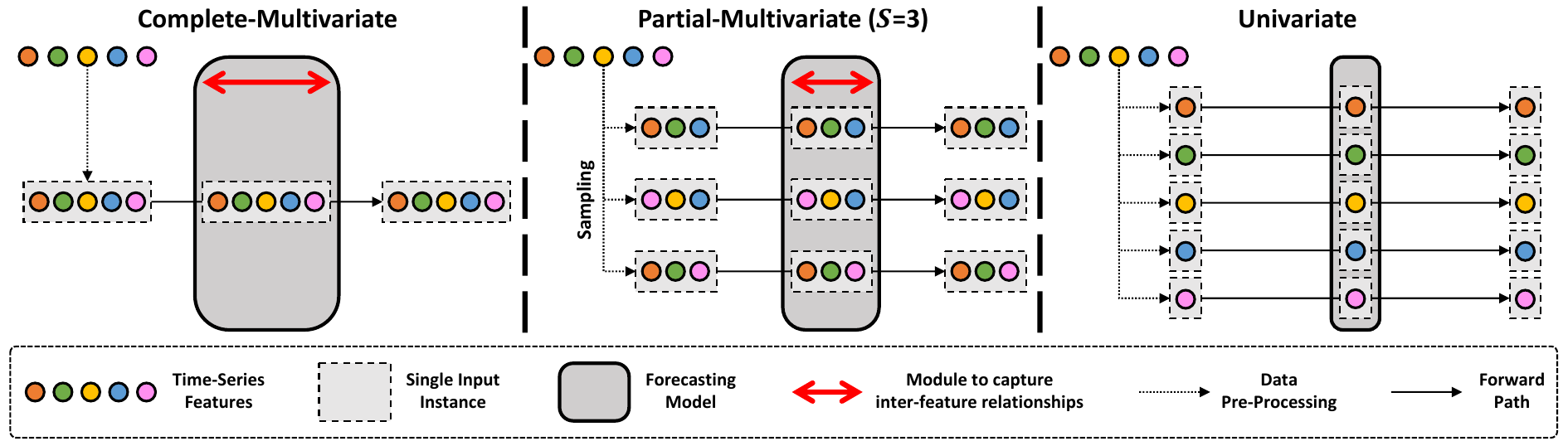}
    \vspace{-15pt}
    \caption{Visualization of three types of models. While     
    the complete-multivariate model processes a complete set of features simultaneously taking into account their relationships, the univariate model, which treats each feature as separate inputs for a shared neural network, disregards relationships. However, in the partial-multivariate model, several subsets of size $S$ are sampled from a complete feature set and relationships are captured only within each subset --- note that a single neural network is shared by all sampled subsets.}
    
    \label{fig:three-variate}
    \vspace{-5pt}
\end{figure*}

\textbf{Contribution.}
Our research goal is to explore a middle ground between these two extreme types, which we call a \emph{Partial-Multivariate} model. Complete-multivariate models process all features simultaneously with an inter-feature module computing their complete relationships. On the other hand, typically in a univariate model, each feature is pre-processed into separate inputs, and a single neural network is shared by all features.~\citep{patchtst,timemixer,pits} In contrast to the two extreme cases, our partial-multivariate model includes a single neural network that captures dependencies within subsets of features (\textit{i.e.}, partial dependencies) and  is shared by all sampled subsets. The differences among the three models are illustrated in Figure~\ref{fig:three-variate}.
To implement the concept of partial-multivariate models, we propose \emph{Partial-Multivariate Transformer}, \method. Inspired by \citet{patchtst}, \method can capture any partial relationship with a shared Transformer by tokenizing features individually and computing attention maps for selected features. Additionally, we propose training algorithms for \method based on random sampling or partitioning, under a usual assumption that the prior knowledge on how to select subsets of features is unavailable.

In experiments, we demonstrate that \method outperforms existing complete-multivariate or univariate models, attaining the best forecasting accuracy against 20 baselines. To explain the superiority of our partial-multivariate model against the other two types, we provide a theoretical analysis based on McAllester’s bound on generalization errors~\citep{generalizationboundmc}, suggesting the following two reasons: 
\textit{(i)} higher entropy in posterior distributions of hypotheses of partial-multivariate models than complete-multivariate ones and \textit{(ii)} the increased training set size for partial-multivariate models against complete-multivariate and univariate ones. 
Our analysis is supported by our empirical results.
On top of that, to improve forecasting performance further, we introduce a simple inference technique for \method based on the fact that the probability that a specific event occurs at least once out of several trials increases as the number of trials gets large. Finally, we show other useful properties of \method: efficient inter-feature attention costs against other Transformers including inter-feature attention modules,  and robustness under missing features compared to complete-multivariate models.
To sum up, our contributions are summarized as follows:
\begin{itemize}
[noitemsep,topsep=0pt,parsep=3pt,partopsep=0pt,leftmargin=10pt]
    \item To the best of our knowledge, for the first time, we introduce the novel concept of \emph{Partial-Multivariate} models in the realm of time-series forecasting, devising Transformer-based \method. To train \method, we propose a train algorithm based on random sampling or partitioning.
    
    \item The abundant experimental results demonstrate that our \method achieves the best performance against 20 baselines. Furthermore, we provide a theoretical analysis for the excellence of our model and provide empirical results supporting this analysis.
    
    \item We propose an inference technique for \method to enhance forecasting accuracy further, inspired by the relationships between the number of trials and probabilities of sampling a specific subset at least once. At last, we discover some useful properties of \method against complete-multivaraite models: efficient inter-feature attention costs and robustness under missing features.
    
\end{itemize}






\section{Related Works}

To solve the forecasting problem with multiple features, it is important to discover not only temporal but also inter-feature relationships. As for inter-feature relationships, existing studies often aim to capture full dependencies among a complete set of features, which we call complete-multivariate models. For example, some approaches encode all features into a single hidden vector, which is then decoded back into feature spaces after some processes. This technique has been applied to various architectures, 
including RNNs~\citep{che2016recurrent}, CNNs~\citep{bai2018empirical}, state-space models~\citep{statespace2}, and Transformers~\citep{autoformer}. 
Conversely, other complete-multivariate studies have developed modules to explicitly capture these relationships. For instance,
\citet{crossformer} computes $D \times D$ attention matrices among $D$ features by encoding each feature into a separate token, while \citet{wu2020connecting} utilizes graph neural networks with graphs of inter-feature relationships.
Additionally, \citet{tsmixer} parameterizes a weight matrix $W \in \mathbb{R}^{D \times D}$, where each element in the $i$-th row and $j$-th column represents the relationship between the $i$-th and $j$-th features.
Unlike complete-multivariate models which fully employ inter-feature information, new models have recently been developed: univariate models.~\citep{dlinear,fits,patchtst,timemixer,pits}
These models capture temporal dynamics but ignore the inter-feature information by processing each of $D$ features separately as independent inputs. It is worth noting that univariate models usually include a single neural network shared by all features.
Existing models typically either utilize inter-feature information among all features or ignore it entirely. 
In contrast, we propose a completely different approach to using inter-feature information, where a shared model captures relationships only within subsets of features, named partial-multivariate models.

\section{Method}
In this section, we first formulate partial-multivariate forecasting models. Subsequently, we introduce \textbf{Partial-Multivariate Transformer (\method)} with its training algorithm and inference technique. Finally, we provide a theoretical analysis to explain the superiority of \method.

\begin{figure*}[t]
    \centering
    \includegraphics[width=\textwidth]{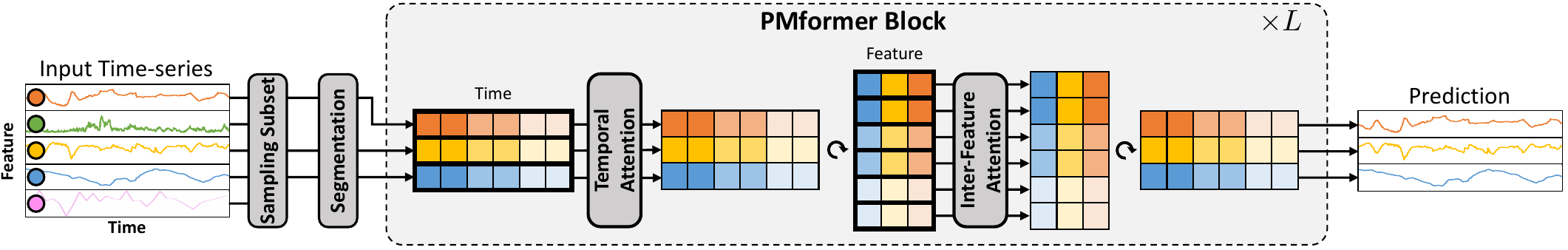}
    \vspace{-15pt}
    \caption{Architecture of Partial-Multivariate Transformer (\method). To emphasize row-wise attention operations, we enclose each row within bold frames before feeding them into the attention modules. In this figure, the subset size $S$ is 3.}
    \label{fig:pmformer_archi}
    \vspace{-5pt}
\end{figure*}

\subsection{Partial-Multivariate Forecasting Model}
In this section, we provide the formulation of the partial-multivariate forecasting model. To simplify a notation, we denote the set of integers from $N$ to $M$ (inclusive of $N$ and exclusive of $M$)
as $[N:M]$ (\textit{i.e.}, $[N:M] \coloneqq \{N, N+1, \dots, M-1\}$). Also, when the collection of numbers is given as indices for vectors or matrices, it indicates selecting all indices within the collection. (\textit{e.g.}, $x_{t=[0,T],d=[0,D]} \coloneqq \{\{x_{t,d}\}_{t \in [0:T]}\}_{d \in [0:D]}$). Let $\mathbf{x}_{t,d} \in \mathbb{R}$ the $t$-th observation of the $d$-th feature, and $\mathbf{x}_{[0:T],d}$ and $\mathbf{x}_{[T:T+\tau],d}$ the $d$-th feature's historical inputs and ground truth of future outputs with $T$ and $\tau$ indicating the length of past and future time steps, respectively. Assuming $D$ denotes the number of features, then a partial-multivariate forecasting model $f$ is formulated as follows:
\begin{align}
    \hat{\mathbf{x}}_{[T:T+\tau],\mathbf{F}} &= f(\mathbf{x}_{[0:T],\mathbf{F}},\mathbf{F}), \quad \quad
    \mathbf{F} \in \mathcal{F}^{all}=\{\mathbf{F}|\mathbf{F} \subset [0:D], |\mathbf{F}|=S\}. \label{eq:partial_multivariate_model}
\end{align}
After sampling a subset $\mathbf{F}$ of size $S$ from a complete set of features $[0:D]$, a model $f$ takes feature indices in $\mathbf{F}$ and their historical observations $\mathbf{x}_{[0:T],\mathbf{F}}$ as input, forecasting the future of the selected features $\hat{\mathbf{x}}_{[T:T+\tau],\mathbf{F}}$. In this formulation, a single model $f$ is shared by any $\mathbf{F} \in \mathcal{F}^{all}$. In Transformer-based models, $\mathbf{F}$ is typically encoded using positional encoding schemes. It is worth noting that from this formulation, univariate and complete-multivariate can be represented by extreme cases of $S$ where $S=1$ and $S=D$, respectively. Our research goal is to explore the middle ground with partial-multivariate models where $1<S<D$.


\subsection{\method}
For complete-multivariate cases, the architectures are required to capture perpetually unchanging (\textit{i.e.}, static) relationships among a complete set of features. In other words, $\mathbf{F}$ in equation~\eqref{eq:partial_multivariate_model} is always the same as $[0:D]$. However, for partial-multivariate cases, $\mathbf{F}$ can vary when re-sampled, requiring to ability to deal with dynamic relationships. Therefore, inspired by recent Transformer-based models using segmentation~\citep{patchtst,crossformer}, we devise \method which addresses this problem by encoding each feature into individual tokens and calculating attention maps only with the feature tokens in $\mathbf{F}$. The overall architecture is illustrated in Figure~\ref{fig:pmformer_archi}.

After sampling $\mathbf{F}$ in equation~\eqref{eq:partial_multivariate_model}, the historical observations of selected features $\mathbf{x}_{[0:T],\mathbf{F}} \in \mathbb{R}^{T \times S}$ are encoded into latent tokens $\mathbf{h}^{(0)} \in \mathbb{R}^{N_S \times S \times d_h}$ via a segmentation process where $N_S$ is the number of segments and $d_h$ is hidden size. The segmentation process is formulated as follows:
\begin{align}
    \mathbf{h}_{b,i}^{(0)} = \mathtt{Linear}(\mathbf{x}_{[\frac{bT}{N_S} : \frac{(b+1)T}{N_S}],\mathbf{F}_i}) + \mathbf{e}^{Time}_{b} + \mathbf{e}^{Feat}_{\mathbf{F}_i}, \quad b \in [0,N_S], \quad i \in [0,S],
\end{align}
where $\mathbf{F}_i$ denotes the $i$-th element in $\mathbf{F}$. A single linear layer maps observations into latent space with learnable time-wise and feature-wise positional embeddings, $\mathbf{e}^{Time} \in \mathbb{R}^{N_S \times d_h}$ and $\mathbf{e}^{Feat} \in  \mathbb{R}^{D \times d_h}$. 
In most scenarios, we can reasonably assume the input time span $T$ to be divisible by $N_S$ by adjusting $T$ during data pre-processing or  padding with zeros as in~\citet{crossformer} and~\citet{patchtst}.

Subsequently, $\mathbf{h}^{(0)}$ is processed through $L$ \method blocks. 
Each block is formulated as follows:
\begin{align}
    \bar{\mathbf{h}}^{(\ell-1)} &=\mathbf{h}^{(\ell-1)}+\mathtt{Feature}\text{-}\mathtt{Attention}(\mathbf{h}^{(\ell-1)},\mathtt{Temporal}\text{-}\mathtt{Attention}(\mathbf{h}^{(\ell-1)})), \label{eq:pmformer_attn} \\
    \mathbf{h}^{(\ell)} &=\bar{\mathbf{h}}^{(\ell-1)}+\mathtt{MLP}(\bar{\mathbf{h}}^{(\ell-1)}),\quad\ell=1,\ldots,L.\label{eq:pmformer_mlp}
\end{align}
 $\mathtt{MLP}$ in equation~\eqref{eq:pmformer_mlp} operates both feature-wise and time-wise, resembling the feed-forward networks found in the original Transformer~\citep{NIPS2017_3f5ee243}. As shown in equation~\eqref{eq:pmformer_attn}, there are two types of attention modules:
\begin{align}
    \forall i \in [0:S], \quad \mathtt{Temporal}\text{-}\mathtt{Attention}(\mathbf{h})_{[0:N_S],i}
    & = \mathtt{MHSA}(\mathbf{h}_{[0:N_S],i},\mathbf{h}_{[0:N_S],i},\mathbf{h}_{[0:N_S],i}),\label{eq:temp_attn} \\
    \forall b \in [0:N_S], \quad \mathtt{Feature}\text{-}\mathtt{Attention}(\mathbf{h}, \mathbf{v})_{b,[0:S]}
    & = \mathtt{MHSA}(\mathbf{h}_{b,[0:S]},\mathbf{h}_{b,[0:S]},\mathbf{v}_{b,[0:S]}).\label{eq:feat_attn}
\end{align}
$\mathtt{MHSA}(\mathbf{Q},\mathbf{K},\mathbf{V})$ denotes the multi-head self-attention layer like in~\citet{NIPS2017_3f5ee243} where $\mathbf{Q},\mathbf{K}$, and $\mathbf{V}$ are queries, keys and values. 
While temporal attention is responsible for capturing temporal dependencies, feature attention mixes representations among features in $\mathbf{F}$. 

Starting with initial representations $\mathbf{h}^{(0)}$, PMformer encoder with $L$ blocks generates final representations $\mathbf{h}^{(L)}$. These representations are then passed through a decoder to forecast future observations. Similar to ~\citet{patchtst}, the concatenated representations $\mathbf{h}^{(L)}_{[0:N_S],i}$ are mapped to future observations $\mathbf{x}_{[T,T+\tau],\mathbf{F}_i}$ via a single linear layer.

\begin{wrapfigure}[15]{r}{0.45\textwidth}
\vspace{-3.2em}
\hspace{0.6em}
\begin{minipage}{.96\linewidth}
\begin{algorithm}[H]
\small
\caption{Training Algorithm 1}\label{alg:train2}
\KwIn{\# of features $D$, \# of subsets $N_U$, Past obs.  $\mathbf{x}_{[0:D]}$, Future obs. $\mathbf{y}_{[0:D]}$}

\While{is\_converge}{
    
    Sample all $\mathbf{F}(g)$ with random partition; \label{alg_line:start_inf}
    
    \For{$g \gets 0$  \KwTo $N_U-1$}{
    
    \If{use\_random\_partition}{
        $\mathbf{F}=\mathbf{F}(g)$
    }\Else{
        Sample $\mathbf{F}$ from [0:D]
    }
    
    $\hat{\mathbf{y}}_{\mathbf{F}} = \mathtt{PMformer}
    (\mathbf{x}_{\mathbf{F}}, \mathbf{F})$; \label{alg_line:end_inf}

    $\text{Loss}_g=\mathtt{MSE}(\hat{\mathbf{y}}_{\mathbf{F}}, {\mathbf{y}}_{\mathbf{F}})$;
    
    }   

    $\text{Loss} = \sum_{g \in [0:N_U]} \text{Loss}_g/N_U$; 

    Train $\mathtt{PMformer}$ with Loss; \label{alg_line:train}
}

\Return Trained $\mathtt{PMformer}$
\end{algorithm}
\end{minipage}
\end{wrapfigure}
\subsection{Training Algorithm for \method{}}
To train \method, the process to sample $\mathbf{F}$ from a complete set of features is necessary. Ideally, the sampling process would select mutually correlated features. 
However, prior knowledge about the relationships between features is usually unavailable. Therefore, we use random sampling where $\mathbf{F}$ is sampled fully at random, leading to training Algorithm~\ref{alg:train2} with $use\_random\_partition={False}$ where $N_U$ is the number of sampled subsets per iteration --- note that for-loop in while-loop can be dealt with parallelly with attention masking techniques. However, this training algorithm may result in redundancy or omission in each iteration, as some features might be selected multiple times while others might never be chosen across the $N_U$ trials.
To address this issue, we propose a training algorithm based on random partitioning (see Algorithm~\ref{alg:train2} with $use\_random\_partition=True$). In this algorithm, $D$ features are partitioned into $N_U = {D}/{S}$ disjoint subsets $\{\mathbf{F}(g)\}_{g \in [0:N_U]}$ \footnote{We assume that $D$ is divisible by $S$. 
If not, we can handle such cases by repeating some features, as explained in Appendix~\ref{appen:non-div}.}
where $\mathbf{F}(g) \subset [0:D], |\mathbf{F}(g)|=S, \bigcap_{g \in [0:N_U]} \mathbf{F}(g)  = \phi, \bigcup_{g \in [0:N_U]} \mathbf{F}(g) = \left[0:D\right]$. This scheme can minimize the redundancy and omission of features in each iteration. 
We adopt the training algorithm based on random partitioning as our main training algorithm. Appendix~\ref{appen:random_sampling_partitioning} provides a comparison of these two algorithms in empirical experiments.



\subsection{Inference Technique for \method}\label{sec:inf_tech}
After training \method, we can measure inference score using Algorithm~\ref{alg:train2} without line~\ref{alg_line:train} where $use\_random\_partition=True$ . During inference time, it is important to group mutually significant features together. Attention scores among all features can provide information on feature relationships, but evaluating these scores results in a high computational cost, $\mathcal{O}(D^2)$.
To avoid this, we propose a simple inference technique that doesn't require such expensive computations. We evaluate future predictions by repeating the inference process based on random partitioning $N_I$ times and averaging $N_I$ outputs to obtain the final outcomes --- note that the inference process based on random partitioning achieves efficient costs because of computing attention within subsets of size $S$.



This technique relies on the principle that the probability of a specific event occurring at least once out of $N_I$ trials increases as $N_I$ grows. Let $P({\mathbf{F}_{*}})=p$ be the probability that we sample a specific subset $\mathbf{F}_{*}$ from all possible cases. The probability of sampling $\mathbf{F}_{*}$ at least once out of $N_I$ trials is $1-(1-p)^{N_I}$. Given that $0\le p \le 1$, $1-(1-p)^{N_I}$ increases as $N_I$ increases. By treating a specific subset $\mathbf{F}_{*}$ as one that includes mutually significant features, our inference technique with a large $N_I$ increases the likelihood of selecting a subset including highly correlated features at least once, thereby improving forecasting performance. This is supported by our empirical results in Section~\ref{sec:analysis}.

\subsection{Theoretical Analysis on \method}\label{sec:theoretical_analysis}
In this section, we provide theoretical reasons for superiority of our partial-multivariate models over complete-multivariate and univariate ones, based on PAC-Bayes framework, similar to other works~\citep{deeptime,generalizationbound2,generalizationbound}. Let a neural network $f$ be a partial-multivariate model which samples subsets $\mathbf{F}$ of $S$ size from a complete set of $D$ features as defined in equation~\eqref{eq:partial_multivariate_model}. Also, $\mathcal{T}$ is a training dataset which consists of $m$ instances sampled from the true data distribution. 
$\mathcal{H}$ denotes the hypothesis class of $f$ with $\mathbf{P}(h)$ and $\mathbf{Q}(h)$ representing the prior and posterior distributions over the hypotheses $h$, respectively. Then, based on~\citet{generalizationboundmc}, the generalization bound of a partial-multivariate model $f$ is given by:
\begin{theorem}
    Under some assumptions, with probability at least $1-\delta$ over the selection of the sample $\mathcal{T}$, we have the following for generalized loss $l(\mathbf{Q})$ under posterior distributions $\mathbf{Q}$.
\begin{align}
    l(\mathbf{Q}) \le \sqrt{\frac{-H(\mathbf{Q}) + \log \frac{1}{\delta} + \frac{5}{2} \log m + 8 + C}{2m - 1}}, \label{eq:generalizationbound}
\end{align}
where $H(\mathbf{Q})$ is the entropy of $\mathbf{Q}$, (\textit{i.e.}, $H(\mathbf{Q}) = E_{h\sim \mathbf{Q}}[-\log \mathbf{Q}(h)]$) and $C$ is a constant.
\label{theorm:generalizationbound}
\end{theorem}

In equation~\eqref{eq:generalizationbound}, the upper bound depends on $m$ and $-H(\mathbf{Q})$, both of which are related to $S$. 
Selecting subsets of $S$ size from $D$ features leads to $|\mathcal{F}^{all}|={\binom{D}{S}}$  possible cases, affecting $m$ (\textit{i.e.}, $m \propto |\mathcal{F}^{all}|$), where $\mathcal{F}^{all}$ is a pool including all possible $\mathbf{F}$. This is because each subset is regarded as a separate instance as in Figure~\ref{fig:three-variate}.
Also, 
the following theorem reveals relationships between $S$ and $-H(\mathbf{Q})$:
%
\begin{theorem}
    Let $H(\mathbf{Q}_{S})$ be the entropy of a posterior distribution $\mathbf{Q}_{S}$ with subset size $S$. For $S_{+}$ and $S_{-}$ satisfying $S_{+} > S_{-}$. 
    $H(\mathbf{Q}_{S_+}) \le H(\mathbf{Q}_{S_-})$.
    \label{theorem:hypothesisspace}
\end{theorem}
Theorem~\ref{theorem:hypothesisspace} is intuitively connected to the fact that capturing dependencies within large subsets of size $S_+$ is usually harder tasks than the case of small $S_-$, because more relationships are captured in the case of $S_+$. As such, the region of hypotheses that satisfies conditions for such hard tasks would be smaller than the one that meets the conditions for a simple task.
In other words, probabilities of a posterior distribution $\mathbf{Q}_{S_+}$ might be centered on a smaller region of hypotheses than $\mathbf{Q}_{S_-}$, leading to decreasing the entropy of $\mathbf{Q}_{S_+}$. We refer the readers to Appendix~\ref{appen:proof} for full proofs and assumptions for the theorems.

Given the unveiled impacts of $S$ on $m$ and $-H(\mathbf{Q})$, we can estimate $S_{*}$ which is $S$ leading to the lowest upper-bound. 
When considering only the influence of $m$, $S_{*}$ is $D/2$, resulting in the largest $|\mathcal{F}^{all}|=\binom{D}{S}$. On the other hand, considering only that of $-H(\mathbf{Q})$, $S_{*}$ is 1, because $-H(\mathbf{Q})$ decrease as $S$ decreases.
Therefore, considering both effects simultaneously, we can think $1 < S_{*} < D/2$, which means partial-multivariate models ($1<S<D$) are better than univariate models ($S=1$) and complete-multivariate ($S=D$) and the best $S_{*}$ is between $1$ and $D/2$. This analysis is supported by our empirical experimental results in Section~\ref{sec:analysis}. As of now, since we do not evaluate $H(\mathbf{Q})$ exactly, we cannot compare the magnitudes of effects by $m$ and $-H(\mathbf{Q})$, leaving it for future work. Nevertheless, our analysis from the sign of correlations between $S$ and two factors in the upper-bound still is of importance in that it aligns with our empirical observations. 

\section{Experiments}
\subsection{Experimental Setup}
We mainly follow the experiment protocol of previous forecasting tasks with multiple features, \citep{informer}. A detailed description of datasets, baselines, and hyperparameters is in Appendix~\ref{appen:detail_exp}.

\textbf{Datasets.} We evaluate \method and other methods on the seven real-world datasets with $D > 1$: (\textit{i-iv}) ETTh1, ETTh2, ETTm1, and ETTm2 ($D=7$), (\textit{v}) Weather ($D=21$), (\textit{vi}) Electricity ($D=321$), and (\textit{vii}) Traffic ($D=862$). For each dataset, four settings are constructed with different forecasting lengths $\tau$, which is in \{96, 192, 336, 720\}. 

\textbf{Baselines}. Various complete-multivariate and univariate models are included in our baselines. For complete-multivariate baselines, we use Crossformer~\citep{crossformer}, FEDformer~\citep{fedformer}, Pyraformer~\citep{pyraformer}, Informer~\citep{informer}, TSMixer~\citep{tsmixer}, MICN~\citep{cnn2}, TimesNet~\citep{timesnet}, and DeepTime~\citep{deeptime}. On the other hand, univariate ones include NLinear, DLinear~\citep{dlinear}, and PatchTST~\citep{patchtst}. Furthermore, we compare \method against concurrent complete-multivariate (ModernTCN~\citep{moderntcn}, JTFT~\citep{jtft}, GCformer~\citep{gcformer}, CARD~\citep{card}, Client~\citep{client}, and PETformer~\citep{petformer}) and univariate models (FITS~\citep{fits}, PITS~\citep{pits}, and TimeMixer~\citep{timemixer}).
According to code accessibility or fair experimental setting with ours, we select these concurrent models.

\textbf{Other settings.} \method is trained with mean squared error (MSE) between ground truth and forecasting outputs. Also, we use MSE and mean absolute error (MAE) as evaluation metrics, and mainly report MSE. The MAE scores of experimental results are available in Appendix~\ref{appen:add_result}. After training each method with five different random seeds, we measure the scores of evaluation metrics in each case and report the average scores. For the subset size $S$, we use $S=3$ for ETT datasets, $S=7$ for Weather, $S=30$ for Electricity, and $S=20$ for Traffic, satisfying $1<S<D/2$. Also, for the inference technique of \method, we set $N_I$ to 3.


\subsection{Forecasting Result}
Table~\ref{tbl:main} shows the test MSE of representative baselines along with the \method. \method outperforms univariate and complete-multivariate baselines in 27 out of 28 tasks and achieves the second place in the remaining one. We also provide visualizations of forecasting results of \method and some baselines in Appendix~\ref{appen:add_visual}, which shows the superiority of \method. 
On top of that, \method is compared to the nine concurrent baselines in Table~\ref{tbl:main_con}. 
\method shows top-1 performance in 10 cases and top-2 in 12 cases out of 12 cases, attaining a 1.167 average rank. 
The scores in Table~\ref{tbl:main} and~\ref{tbl:main_con} are measured with $N_I=3$, and in Appendix~\ref{appen:ni1}, we provide another forecasting result which shows that our \method still outperforms other baselines even with $N_I=1$.


\begin{table}[t]
    \centering
    \caption{MSE scores of main forecasting results. The best score in each experimental setting is in boldface and the second best is underlined.}\label{tbl:main}
    \vspace{-3pt}
    \scriptsize
    \renewcommand{\arraystretch}{0.9}
    \setlength{\tabcolsep}{0.9pt}
    \centering
    \resizebox{\textwidth}{!}{
    \begin{tabular}{cc|c|ccc|cccccccc}
    \toprule
        \multicolumn{2}{c|}{\multirow{2}{*}{Data}} & Partial-Multivariate & \multicolumn{3}{c|}{Univariate} & \multicolumn{8}{c}{Complete-Multivariate} \\ 
        ~ & ~ & \method & PatchTST & Dlinear & Nlinear & Crossformer & FEDformer & Informer & Pyraformer & TSMixer & DeepTime & MICN & TimesNet \\
        \midrule
\multirow{4}{*}{\rotatebox{90}{ETTh1}} & $\tau$=96 & \textbf{0.361} & \underline{0.370} & 0.375 & 0.374 & 0.427 & 0.376 & 0.941 & 0.664 & \textbf{0.361} & 0.372 & 0.828 & 0.465\\
& 192 & \textbf{0.396} & 0.413 & 0.405 & 0.408 & 0.537 & 0.423 & 1.007 & 0.790 & \underline{0.404} & 0.405 & 0.765 & 0.493\\
& 336 & \textbf{0.400} & 0.422 & 0.439 & 0.429 & 0.651 & 0.444 & 1.038 & 0.891 & \underline{0.420} & 0.437 & 0.904 & 0.456\\
& 720 & \textbf{0.412} & 0.447 & 0.472 & \underline{0.440} & 0.664 & 0.469 & 1.144 & 0.963 & 0.463 & 0.477 & 1.192 & 0.533\\
    \midrule

\multirow{4}{*}{\rotatebox{90}{ETTh2}} & 96 & \textbf{0.269} & \underline{0.274} & 0.289 & 0.277 & 0.720 & 0.332 & 1.549 & 0.645 & \underline{0.274} & 0.291 & 0.452 & 0.381\\
& 192 & \textbf{0.323} & 0.341 & 0.383 & 0.344 & 1.121 & 0.407 & 3.792 & 0.788 & \underline{0.339} & 0.403 & 0.554 & 0.416\\
& 336 & \textbf{0.317} & \underline{0.329} & 0.448 & 0.357 & 1.524 & 0.400 & 4.215 & 0.907 & 0.361 & 0.466 & 0.582 & 0.363\\
& 720 & \textbf{0.370} & 0.379 & 0.605 & 0.394 & 3.106 & 0.412 & 3.656 & 0.963 & 0.445 & 0.576 & 0.869 & \underline{0.371}\\
    \midrule
\multirow{4}{*}{\rotatebox{90}{ETTm1}} & 96 & \textbf{0.282} & 0.293 & 0.299 & 0.306 & 0.336 & 0.326 & 0.626 & 0.543 & \underline{0.285} & 0.311 & 0.406 & 0.343\\
& 192 & \textbf{0.325} & 0.333 & 0.335 & 0.349 & 0.387 & 0.365 & 0.725 & 0.557 & \underline{0.327} & 0.339 & 0.500 & 0.381\\
& 336 & \textbf{0.352} & 0.369 & 0.369 & 0.375 & 0.431 & 0.392 & 1.005 & 0.754 & \underline{0.356} & 0.366 & 0.580 & 0.436\\
& 720 & \underline{0.401} & 0.416 & 0.425 & 0.433 & 0.555 & 0.446 & 1.133 & 0.908 & 0.419 & \textbf{0.400} & 0.607 & 0.527\\
    \midrule
\multirow{4}{*}{\rotatebox{90}{ETTm2}} & 96 & \textbf{0.160} & 0.166 & 0.167 & 0.167 & 0.338 & 0.180 & 0.355 & 0.435 & \underline{0.163} & 0.165 & 0.238 & 0.218\\
& 192 & \textbf{0.213} & 0.223 & 0.224 & 0.221 & 0.567 & 0.252 & 0.595 & 0.730 & \underline{0.216} & 0.222 & 0.302 & 0.282\\
& 336 & \textbf{0.262} & 0.274 & 0.281 & 0.274 & 1.050 & 0.324 & 1.270 & 1.201 & \underline{0.268} & 0.278 & 0.447 & 0.378\\
& 720 & \textbf{0.336} & \underline{0.361} & 0.397 & 0.368 & 2.049 & 0.410 & 3.001 & 3.625 & 0.420 & 0.369 & 0.549 & 0.444\\
    \midrule
\multirow{4}{*}{\rotatebox{90}{Weather}} & 96 & \textbf{0.142} & 0.149 & 0.176 & 0.182 & 0.150 & 0.238 & 0.354 & 0.896 & \underline{0.145} & 0.169 & 0.188 & 0.179\\
& 192 & \textbf{0.185} & 0.194 & 0.220 & 0.225 & 0.200 & 0.275 & 0.419 & 0.622 & \underline{0.191} & 0.211 & 0.231 & 0.230\\
& 336 & \textbf{0.235} & 0.245 & 0.265 & 0.271 & 0.263 & 0.339 & 0.583 & 0.739 & \underline{0.242} & 0.255 & 0.280 & 0.276\\
& 720 & \textbf{0.305} & 0.314 & 0.323 & 0.338 & \underline{0.310} & 0.389 & 0.916 & 1.004 & 0.320 & 0.318 & 0.358 & 0.347\\
    \midrule
\multirow{4}{*}{\rotatebox{90}{Electricity}} & 96 & \textbf{0.125} & \underline{0.129} & 0.140 & 0.141 & 0.135 & 0.186 & 0.304 & 0.386 & 0.131 & 0.139 & 0.177 & 0.186\\
& 192 & \textbf{0.142} & \underline{0.147} & 0.153 & 0.154 & 0.158 & 0.197 & 0.327 & 0.386 & 0.151 & 0.154 & 0.195 & 0.208\\
& 336 & \textbf{0.154} & 0.163 & 0.169 & 0.171 & 0.177 & 0.213 & 0.333 & 0.378 & \underline{0.161} & 0.169 & 0.213 & 0.210\\
& 720 & \textbf{0.176} & \underline{0.197} & 0.203 & 0.210 & 0.222 & 0.233 & 0.351 & 0.376 & \underline{0.197} & 0.201 & 0.204 & 0.231\\
    \midrule
\multirow{4}{*}{\rotatebox{90}{Traffic}} & 96 & \textbf{0.345} & \underline{0.360} & 0.410 & 0.410 & 0.481 & 0.576 & 0.733 & 2.085 & 0.376 & 0.401 & 0.489 & 0.599\\
& 192 & \textbf{0.370} & \underline{0.379} & 0.423 & 0.423 & 0.509 & 0.610 & 0.777 & 0.867 & 0.397 & 0.413 & 0.493 & 0.612\\
& 336 & \textbf{0.385} & \underline{0.392} & 0.436 & 0.435 & 0.534 & 0.608 & 0.776 & 0.869 & 0.413 & 0.425 & 0.496 & 0.618\\
& 720 & \textbf{0.426} & \underline{0.432} & 0.466 & 0.464 & 0.585 & 0.621 & 0.827 & 0.881 & 0.444 & 0.462 & 0.520 & 0.654\\
        \midrule
        \multicolumn{2}{c|}{Avg. Rank} & \textbf{1.036} & 2.893 & 5.357 & 5.071 & 8.036 & 7.821 & 11.429 & 11.286 & \underline{2.821} & 4.607 & 9.000 & 8.179 \\  
        \bottomrule
    \end{tabular}}
\end{table}

\begin{table}[t]
    \scriptsize
    \renewcommand{\arraystretch}{0.9}
    \setlength{\tabcolsep}{3.5pt}
    \centering
    \caption{Test MSE of \method compared to concurrent models.}\label{tbl:main_con}
    \vspace{-3pt}
    \resizebox{\textwidth}{!}{
    \begin{tabular}{cc|cccc|cccc|cccc|c}
    \toprule
        \multicolumn{2}{c|}{\multirow{2}{*}{Method}} & \multicolumn{4}{c|}{ETTm2} & \multicolumn{4}{c|}{Weather} & \multicolumn{4}{c|}{Electricity} & \multirow{2}{*}{\makecell{Avg. \\ Rank}} \\ 
        & & $\tau$=96 & 192 & 336 & 720 & 96 & 192 & 336 & 720 & 96 & 192 & 336 & 720 \\ 
        \midrule
        Partial-Multivariate & \method & \underline{0.160} & \textbf{0.213} & \textbf{0.262} & \textbf{0.336} & \textbf{0.142} & \textbf{0.185} & \textbf{0.235} & \textbf{0.305} & \textbf{0.125} & \underline{0.142} & \textbf{0.154} & \textbf{0.176} & \textbf{1.167} \\ \midrule
\multirow{3}{*}{Univariate} & PITS & 0.163 & 0.215 & \underline{0.266} & \underline{0.342} & 0.154 & 0.191 & 0.245 & 0.309 & 0.132 & 0.147 & 0.162 & 0.199 & 6.000\\
& FITS & 0.164 & 0.217 & 0.269 & 0.347 & 0.145 & 0.188 & \underline{0.236} & 0.308 & 0.135 & \underline{0.142} & 0.163 & 0.200 & 5.083 \\
& TimeMixer & 0.164 & 0.223 & 0.279 & 0.359 & 0.147 & 0.189 & 0.241 & 0.310 & 0.129 & \textbf{0.140} & 0.161 & 0.194 & 6.083\\ \midrule

\multirow{6}{*}{Complete-Multivariate} & JTFT & 0.164 & 0.219 & 0.272 & 0.353 & \underline{0.144} & \underline{0.186} & 0.237 & \underline{0.307} & 0.131 & 0.144 & \underline{0.159} & 0.186 & 4.333 \\
& GCformer & 0.163 & 0.217 & 0.268 & 0.351 & 0.145 & 0.187 & 0.244 & 0.311 & 0.132 & 0.152 & 0.168 & 0.214 & 6.083 \\
& CARD & \textbf{0.159} & \underline{0.214} & \underline{0.266} & 0.379 & 0.145 & 0.187 & 0.238 & 0.308 & 0.129 & 0.154 & 0.161 & \underline{0.185} & \underline{3.917}	\\
& Client & 0.167 & 0.220 & 0.268 & 0.356 & 0.153 & 0.195 & 0.246 & 0.314 & 0.131 & 0.153 & 0.170 & 0.200 & 8.250 \\
& PETformer & \underline{0.160} & 0.217 & 0.274 & 0.345 & 0.146 & 0.190 & 0.241 & 0.314 & \underline{0.128} & 0.144 & \underline{0.159} & 0.195 & 5.000 \\
& ModernTCN & 0.166 & 0.222 & 0.272 & 0.351 & 0.149 & 0.196 & 0.238 & 0.314 & 0.129 & 0.143 & 0.161 & 0.191 & 6.250 \\ 
        \bottomrule
    \end{tabular}
    }
\end{table}

\subsection{Analysis}\label{sec:analysis}

In this section, we provide some analysis on our \method. We refer the readers to Appendix~\ref{appen:add} for additional experimental results.

\textbf{Empirical result supporting the theoretical analysis.} In Section~\ref{sec:theoretical_analysis}, we think that $S_{*}$ leading to the best forecasting performance is between $1$ and $D/2$. To validate this analysis, we provide Table~\ref{tbl:ablation}, which shows that partial-multivariate settings ($1<S<D$) outperform others with $S=1$ or $D$, in most cases. On top of that, our analysis is further supported by the U-shaped plots in Figure~\ref{fig:change_s} where the best MSE is achieved when $1<S<D/2$ and the worst one is in $S \in \{1,D\}$. 

On top of that, we conduct another experiment in Figure~\ref{fig:change_f_all} where we adjust the size of subsets' pool ($|\mathcal{F}^{all}|$) while fixing $S$. For training of original \method, $\mathcal{F}^{all}$ consists of all possible subsets, leading to $|\mathcal{F}^{all}|=\binom{D}{S}$. However, for this experiment, we limit $|\mathcal{F}^{all}|$ into $\alpha \times N_U$ by randomly removing some subsets from all possible cases where $N_U$ is the number of sampling a subset in each iteration and $\alpha \in \{1,400,1600,6400,\text{Max}\}$. `Max' denotes $\alpha$ leading to  $|\mathcal{F}^{all}|=\binom{D}{S}$.
Figure~\ref{fig:change_f_all} shows that as $|\mathcal{F}^{all}|$ increases, forecasting performance improves. These experimental results align with our theoretical analyses that $|\mathcal{F}^{all}|$ is proportional to a training set size $m$ and large $m$ leads to low upper-bounds on generalization errors.




\begin{table}[t]
    \centering
    \caption{Comparison among three types of models by adjusting $S$ in \method.}\label{tbl:ablation}
    \vspace{-3pt}
    \scriptsize
    \footnotesize
    \renewcommand{\arraystretch}{0.9}
    \setlength{\tabcolsep}{2pt}
    \resizebox{\textwidth}{!}{
    \begin{tabular}{cc|cccc|cccc|cccc|cccc}
    \toprule
        \multicolumn{2}{c|}{\multirow{2}{*}{\makecell{\method{} Variants}}} & \multicolumn{4}{c|}{ETTh2 ($D=$ 7)} & \multicolumn{4}{c|}{Weather ($D=$ 21)} & \multicolumn{4}{c|}{Electricity ($D=$ 321)} & \multicolumn{4}{c}{Traffic ($D=$ 862)} \\ 
         & & $\tau$=96 & 192 & 336 & 720 & 96 & 192 & 336 & 720 & 96 & 192 & 336 & 720 & 96 & 192 & 336 & 720 \\ 
        \midrule
        Univariate & $S=1$ & \underline{0.272} & \underline{0.325} & \underline{0.318} & 0.374  & \textbf{0.141} & \underline{0.186} & \underline{0.237} & 0.308 & \underline{0.128} & \underline{0.146} & \underline{0.163} & \underline{0.204} & 0.368 & 0.388 & 0.404 & \underline{0.441} \\ 
        \textbf{Partial-Multivariate} &  $1<S<D$ & \textbf{0.269} & \textbf{0.323} & \textbf{0.317} & \textbf{0.370} & \underline{0.142} & \textbf{0.185} & \textbf{0.235} & \textbf{0.305} & \textbf{0.125} & \textbf{0.142} & \textbf{0.154} & \textbf{0.176} & \textbf{0.345} & \textbf{0.370} & \textbf{0.385} & \textbf{0.426} \\ 
         Complete-Multivariate & $S=D$ & \textbf{0.269} & \underline{0.325} & \underline{0.318} & \underline{0.371} & 0.146 & 0.192 & 0.244  & \underline{0.307} & 0.129 & 0.147 & \underline{0.163} & \underline{0.204} & \underline{0.363} & \underline{0.383} & \underline{0.394} & \underline{0.441} \\ 
        \bottomrule
    \end{tabular}}
\end{table}

\begin{figure}[!t]
\centering
\begin{minipage}{0.49\linewidth}
\centering
    \includegraphics[width=\linewidth]{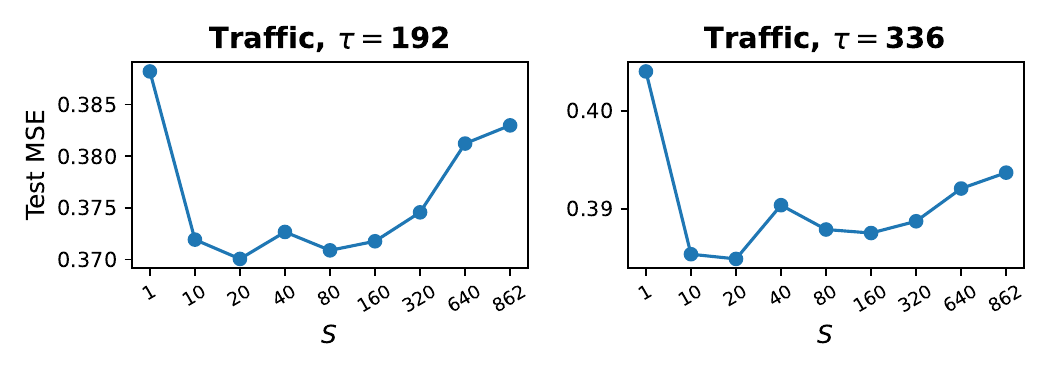}
    \vspace{-20pt}
    \caption{Test MSE by changing $S$.}
    \label{fig:change_s}
\end{minipage}
\hfill
\begin{minipage}{0.49\linewidth}
\centering
    \includegraphics[width=\linewidth]{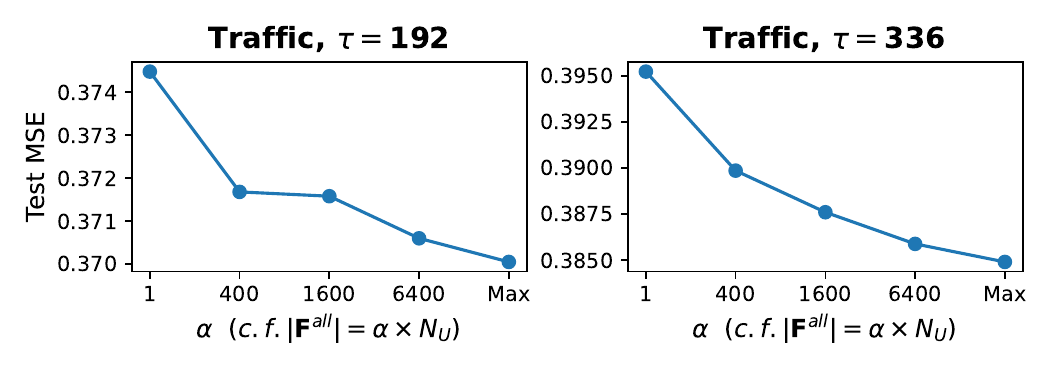}
    \vspace{-20pt}
    \caption{Test MSE by changing $|\mathbf{F}^{all}|$, fixing $S$.}
    \label{fig:change_f_all}
\end{minipage}
\end{figure}




\begin{figure}[!t]
    \centering
    \subfigure[Sensitivity to $N_I$]{
        \includegraphics[width=0.49\columnwidth]{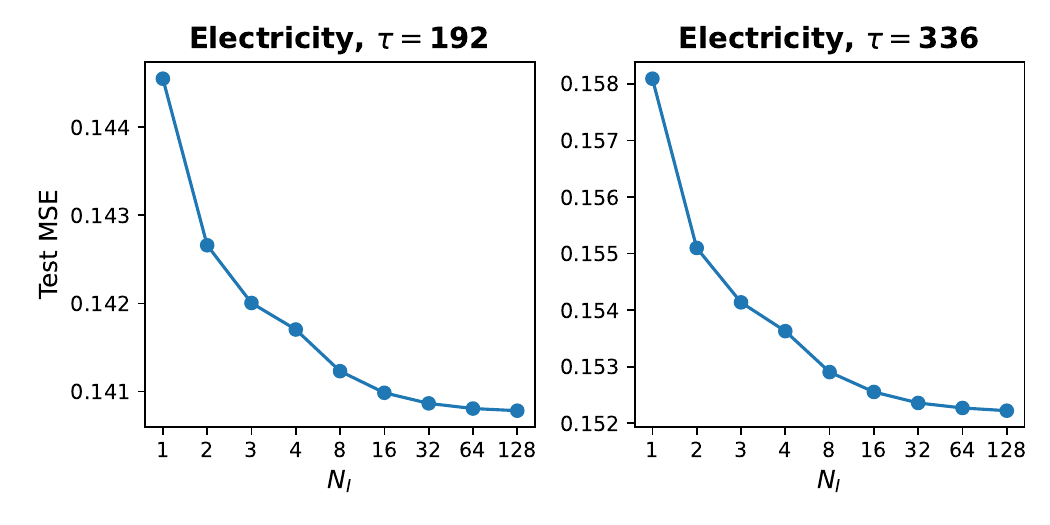}}
    \subfigure[Changes in the effect of $N_I$ when $S$ increases]{
    \includegraphics[width=0.49\columnwidth]{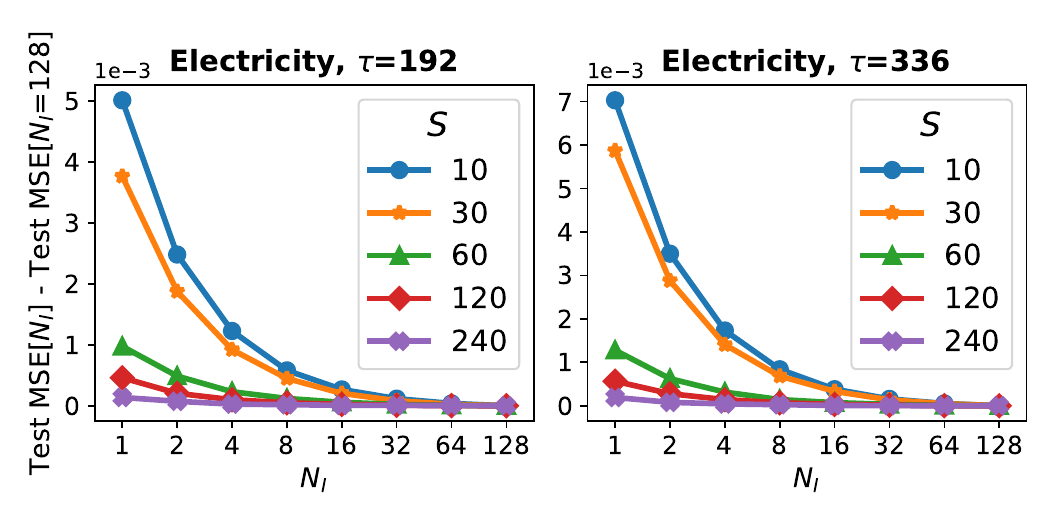}}
    \vspace{-10pt}
    \caption{The effect of $N_I$ on test MSE when (a) $S$ is fixed to the selected hyperparameter and (b) $S$ changes. For (b), the y axis shows the difference of test MSE between when $N_E \in \{1,2,4,8,16,32,64,128\}$ and $N_E=128$.}
    \label{fig:sens_to_ni}
\end{figure}


\begin{table}[!t]
    \centering
    \caption{Test MSE of \method with various inference techniques --- note that all variants of \method are trained with the same algorithms as ours. To identify relevance (significance) of features to others, we utilize attention scores.}\label{tbl:various_inference}
    \vspace{-3pt}
    \scriptsize
    \renewcommand{\arraystretch}{0.9}
    \setlength{\tabcolsep}{6.5pt}
    \resizebox{\textwidth}{!}{
    \begin{tabular}{c|cccc|cccc}
    \toprule
        \multirow{2}{*}{\makecell{Inference Technique}} & \multicolumn{4}{c|}{Electricity ($D=$ 321)} & \multicolumn{4}{c}{Traffic ($D=$ 862)} \\ 
        & $\tau$=96 & 192 & 336 & 720 & 96 & 192 & 336 & 720 \\ 
        \midrule
        Proposed Technique with $N_I=3$ (Ours) & \textbf{0.125} & \textbf{0.142} & \textbf{0.154} & \textbf{0.176} & \textbf{0.345} & \textbf{0.370} & \textbf{0.385} & \textbf{0.426} \\ 
        Sampling A Subset of Mutually Significant Features & \underline{0.132} & \underline{0.148} & 0.178 & \underline{0.205} & \underline{0.352} & \underline{0.372} & \underline{0.386} & \underline{0.428} \\
        Sampling A Subset of Mutually Insignificant Features & 0.135 & 0.167 & \underline{0.174} & 0.235 & 0.377 & 0.410 & 0.410 & 0.444\\
        \bottomrule
    \end{tabular}}
    \vspace{-5pt}
\end{table}


\begin{figure}[t]
\centering
\begin{minipage}{0.49\linewidth}
\centering
    \includegraphics[width=\linewidth]{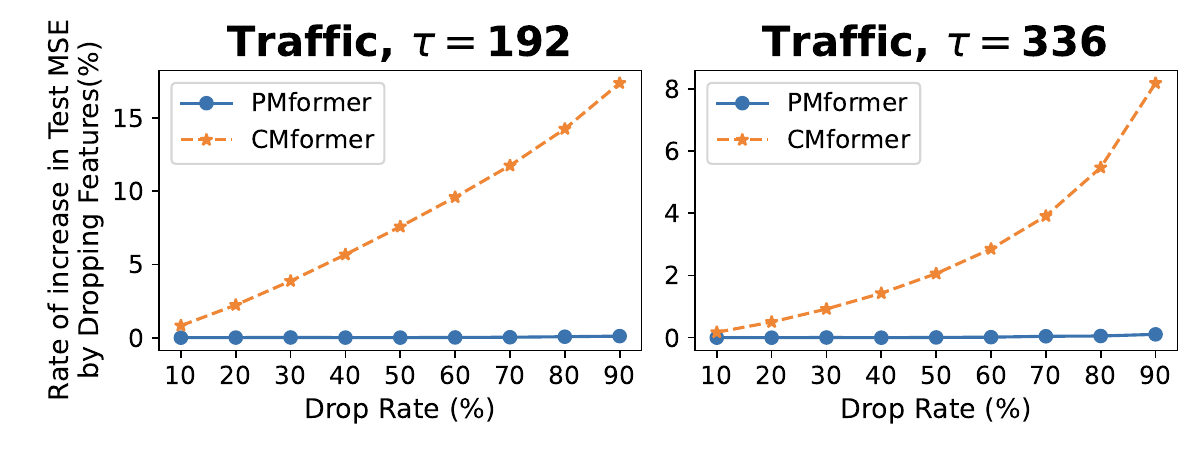}
    \vspace{-20pt}
    \caption{Increasing rate of test MSE by dropping $n$\% features in \method or Complete-Multivariate Transformer (CMformer).}
    \label{fig:robust_to_drop}
\end{minipage}
\hfill
\begin{minipage}{0.49\linewidth}
\centering
    \includegraphics[width=\linewidth]{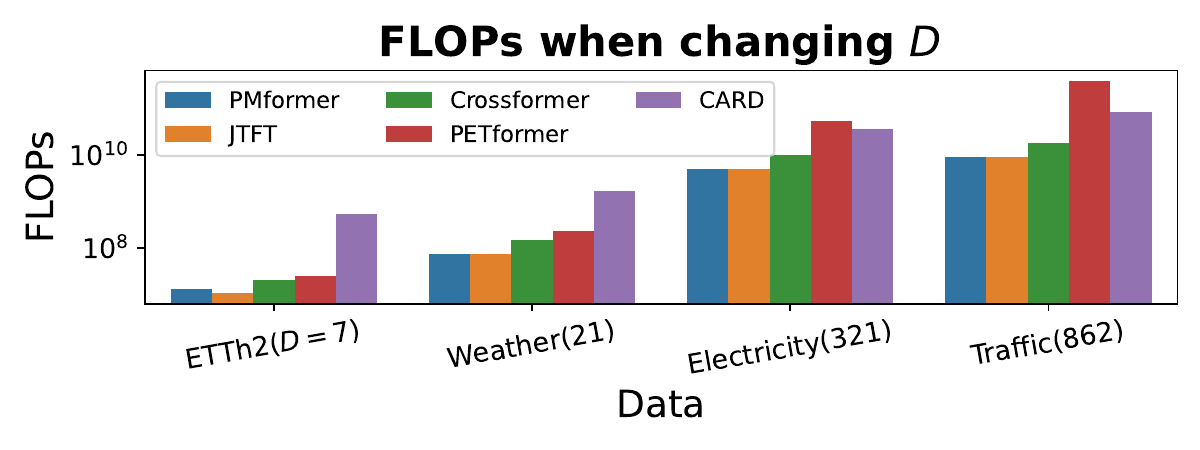}
    \vspace{-20pt}
    \caption{FLOPs of self-attention for inter-feature dependencies in various  Transformers when changing $D$.}
    \label{fig:flops}
\end{minipage}
\vspace{-5pt}
\end{figure}

\textbf{Analysis on the inference technique.}
In Section~\ref{sec:inf_tech}, we think that large $N_I$ (the repeating number of the inference process based on random partitioning) would improve forecasting results by increasing the probabilities that sampled subsets include mutually significant features at least once out of $N_I$ trials. Figure~\ref{fig:sens_to_ni}(a) is aligned with our thought, showing monotonically decreasing test MSE as $N_I$ gets large.  
In Figure~\ref{fig:sens_to_ni}(b), we investigate relationships between the feature subset size $S$ and $N_I$ by measuring performance gain by increasing $N_I$ in various $S$. This figure shows that the effect of increasing $N_I$ tends to be smaller, as $S$ increases. We think this is because a single subset $\mathbf{F}$ with large $S$ can contain a number of features, so mutually significant features can be included in such large subsets at least once only with few repetitions.

Besides the inference technique based on random selection, we explore another technique which samples subsets of mutually important features by selecting some keys with the highest attention scores per query. We compare this technique to the counterpart which selects keys based on the lowest attention score. In Table~\ref{tbl:various_inference}, we provide the forecasting MSE of each inference technique. 
--- note that only the inference method is different while the training algorithm remains the same as the original one in Algorithm~\ref{alg:train2}. 
In that an inference technique utilizing the highest attention scores outperforms one with the lowest ones, attention scores are helpful in identifying relationships between features to some extent. However, our proposed method based on random partitioning achieves the best forecasting performance. Furthermore, identifying relationships between features requires high-cost attention computation which calculates  full attention between $D$ features, leading to $\mathcal{O}(D^2)$. On the other hand, our proposed inference technique doesn't incorporate such high-cost computations but just repeats low-cost ones $N_I$ times, each of which has $O(SD)$ costs for computing inter-feature relationships --- note that $S < D/2$. In Appendix~\ref{appen:complexity}, we elaborate on the details of why \method achieves $\mathcal{O}(SD)$. 
Therefore, against the inference technique with information of attention scores, our proposed one with random partitioning is superior in terms of efficiency and forecasting accuracy.

 

\textbf{Other Advantages of \method.}
In the real world, some features in time series are often missing. Inspired by the works that address irregular time series where observations at some time steps~\citep{che2016recurrent, kidger2020ode} are missing, we randomly drop some features of input time series in the inference stage and measure the increasing rate of test MSE in undropped features. For comparison, we use the original \method{} and a complete-multivariate version of \method (CMformer) by setting $S$ to $D$. \method{} can address the missingness by simply excluding missing features in the random sampling process, while CMformer has no choice but to pad dropped features with zeros. In Figure~\ref{fig:robust_to_drop}, unlike the other case, \method{} maintains its forecasting performance, regardless of the drop rate of the features. This robust characteristic gives \method{} more applicability in real-world situations where some features are not available.

For Transformers with inter-feature attention modules (\method, Crossformer, JTFT, PETformer, and CARD), we compare the costs of their inter-feature modules using floating point operations (FLOPs) in Figure~\ref{fig:flops}. When na\"ively computing inter-feature attention like PETformer, the attention cost is $\mathcal{O}(D^2)$ where $D$ is the number of features. On the other hand, due to capturing only partial relationships, the attention cost of \method is reduced to $\mathcal{O}(SD)$ where $S$ is the size of each subset. In Appendix~\ref{appen:complexity}, we elaborate on the details of the reason why the inter-feature module in \method achieves $\mathcal{O}(SD)$. Given that small $S$ is enough to generate good forecasting performance (\textit{e.g.}, $S$ = 20$\sim$30 for 300$\sim$800 features), the attention cost is empirically efficient. As a result, \method achieves the lowest or the second lowest FLOPs compared to others, as shown in Figure~\ref{fig:flops}. Although Crossformer, JTFT, and CARD achieve $O(RD)$ complexities of inter-feature attention with low-rank approximations where $R$ is the rank, our \method shows quite efficient costs, compared to them.

\section{Conclusion}\label{sec:conclusion}

Various models have been developed to address the forecasting problem with multiple variables. However, most studies focus on two extremes: univariate or complete-multivariate models. To explore the middle ground between them, our research introduces a novel concept, partial-multivariate models, devising \method.
\method captures dependencies only within subsets of a complete feature set using a single inter-feature attention module shared by all subsets. To train \method under usual situations without prior knowledge on subset selection, we propose a training algorithm based on random sampling or partitioning. Extensive experiments show that \method outperforms 20 baseline models. To explain \method's superior performance, we theoretically analyze the upper-bound on generalization errors of \method compared to complete-multivariate and univariate ones and provide empirical results  supporting the results of the theoretical analysis.
Additionally, we enhance forecasting accuracy by introducing a simple inference technique for \method. Finally, we highlight \method's useful characteristics in terms of the efficiency of inter-feature attention and robustness under missing features against complete-multivariate models.


\textbf{Limitation.} Further theoretical analysis is needed to more accurately explain partial-multivariate models, such as precisely calculating the entropy of posterior distributions and relaxing certain assumptions. Despite these limitations, our research still remains significant as it introduces the concept of partial-multivariate models for the first time and provides theoretical analyses that align with empirical results.









\textbf{Broader Impacts.} Our work might have positive effects by benefiting those who devise foundation models for times series because different time series vary in the number of features and our feature sampling scheme where the sampled subset size is always $S$ can overcome this heterogeneity. As for the negative ones, we think the negative effects of forecasting well are still under-explored.




%


\clearpage

\bibliographystyle{rusnat}
\bibliography{bib_file}

\begin{thebibliography}{52}
\providecommand{\natexlab}[1]{#1}
\providecommand{\EM}{\em}
\providecommand{\RNtxt}{\relax}
\RNtxt{}

\bibitem[Amit, Meir(2019)R.~Amit, R.~Meir]{generalizationbound2}
{\EM Amit Ron, Meir Ron}.
\newblock Meta-Learning by Adjusting Priors Based on Extended PAC-Bayes Theory. 2019.

\bibitem[Angryk et~al.(2020)R.~A. Angryk, P.~C. Martens, B.~Aydin, D.~Kempton, S.~S. Mahajan, S.~Basodi, A.~Ahmadzadeh, X.~Cai, S.~Filali~Boubrahimi, S.~M. Hamdi, et~al.]{angryk2020multivariate}
{\EM Angryk Rafal~A, Martens Petrus~C, Aydin Berkay, Kempton Dustin, Mahajan Sushant~S, Basodi Sunitha, Ahmadzadeh Azim, Cai Xumin, Filali~Boubrahimi Soukaina, Hamdi Shah~Muhammad, others }.
\newblock Multivariate time series dataset for space weather data analytics \allowbreak\newblock// Scientific data. 2020. 7, 1. 227.

\bibitem[Bai et~al.(2018)S.~Bai, J.~Z. Kolter, V.~Koltun]{bai2018empirical}
{\EM Bai Shaojie, Kolter J.~Zico, Koltun Vladlen}.
\newblock An Empirical Evaluation of Generic Convolutional and Recurrent Networks for Sequence Modeling. 2018.

\bibitem[Bhanja, Das(2019)S.~Bhanja, A.~Das]{bhanja2019impact}
{\EM Bhanja Samit, Das Abhishek}.
\newblock Impact of Data Normalization on Deep Neural Network for Time Series Forecasting. 2019.

\bibitem[Che et~al.(2016)Z.~Che, S.~Purushotham, K.~Cho, D.~Sontag, Y.~Liu]{che2016recurrent}
{\EM Che Zhengping, Purushotham Sanjay, Cho Kyunghyun, Sontag David, Liu Yan}.
\newblock Recurrent Neural Networks for Multivariate Time Series with Missing Values. 2016.

\bibitem[Chen et~al.(2023{\natexlab{a}})S.-A. Chen, C.-L. Li, N.~Yoder, S.~O. Arik, T.~Pfister]{tsmixer}
{\EM Chen Si-An, Li~Chun-Liang, Yoder Nate, Arik Sercan~O., Pfister Tomas}.
\newblock TSMixer: An all-MLP Architecture for Time Series Forecasting. 2023{\natexlab{a}}.

\bibitem[Chen et~al.(2023{\natexlab{b}})Y.~Chen, S.~Liu, J.~Yang, H.~Jing, W.~Zhao, G.~Yang]{jtft}
{\EM Chen Yushu, Liu Shengzhuo, Yang Jinzhe, Jing Hao, Zhao Wenlai, Yang Guangwen}.
\newblock A Joint Time-frequency Domain Transformer for Multivariate Time Series Forecasting. 2023{\natexlab{b}}.

\bibitem[Cybenko(1989)G.~Cybenko]{Cybenko1989}
{\EM Cybenko G.}
\newblock Approximation by superpositions of a sigmoidal function \allowbreak\newblock// Mathematics of Control, Signals and Systems. Dec 1989. 2, 4. 303--314.

\bibitem[Domingos(2012)P.~Domingos]{10.1145/2347736.2347755}
{\EM Domingos Pedro}.
\newblock A few useful things to know about machine learning \allowbreak\newblock// Commun. ACM. oct 2012. 55, 10. 78–87.

\bibitem[Du et~al.(2021)Y.~Du, J.~Wang, W.~Feng, S.~Pan, T.~Qin, R.~Xu, C.~Wang]{rnn2}
{\EM Du~Yuntao, Wang Jindong, Feng Wenjie, Pan Sinno, Qin Tao, Xu~Renjun, Wang Chongjun}.
\newblock AdaRNN: Adaptive Learning and Forecasting of Time Series. 2021.

\bibitem[Gao et~al.(2023)J.~Gao, W.~Hu, Y.~Chen]{client}
{\EM Gao Jiaxin, Hu~Wenbo, Chen Yuntian}.
\newblock Client: Cross-variable Linear Integrated Enhanced Transformer for Multivariate Long-Term Time Series Forecasting. 2023.

\bibitem[Gu et~al.(2022)A.~Gu, K.~Goel, C.~Ré]{statespace2}
{\EM Gu~Albert, Goel Karan, Ré Christopher}.
\newblock Efficiently Modeling Long Sequences with Structured State Spaces. 2022.

\bibitem[Jones et~al.(2009)S.~S. Jones, R.~S. Evans, T.~L. Allen, A.~Thomas, P.~J. Haug, S.~J. Welch, G.~L. Snow]{jones2009multivariate}
{\EM Jones Spencer~S, Evans R~Scott, Allen Todd~L, Thomas Alun, Haug Peter~J, Welch Shari~J, Snow Gregory~L}.
\newblock A multivariate time series approach to modeling and forecasting demand in the emergency department \allowbreak\newblock// Journal of biomedical informatics. 2009. 42, 1. 123--139.

\bibitem[Kidger et~al.(2020)P.~Kidger, J.~Morrill, J.~Foster, T.~Lyons]{kidger2020ode}
{\EM Kidger Patrick, Morrill James, Foster James, Lyons Terry}.
\newblock Neural Controlled Differential Equations for Irregular Time Series \allowbreak\newblock// Advances in Neural Information Processing Systems. 2020.

\bibitem[Lee et~al.(2024)S.~Lee, T.~Park, K.~Lee]{pits}
{\EM Lee Seunghan, Park Taeyoung, Lee Kibok}.
\newblock Learning to Embed Time Series Patches Independently \allowbreak\newblock// The Twelfth International Conference on Learning Representations. 2024.

\bibitem[Li et~al.(2020)S.~Li, X.~Jin, Y.~Xuan, X.~Zhou, W.~Chen, Y.-X. Wang, X.~Yan]{logtrans}
{\EM Li~Shiyang, Jin Xiaoyong, Xuan Yao, Zhou Xiyou, Chen Wenhu, Wang Yu-Xiang, Yan Xifeng}.
\newblock Enhancing the Locality and Breaking the Memory Bottleneck of Transformer on Time Series Forecasting. 2020.

\bibitem[Liang et~al.(2024)A.~Liang, X.~Jiang, Y.~Sun, C.~Lu]{liang2024bimamba4ts}
{\EM Liang Aobo, Jiang Xingguo, Sun Yan, Lu~Chang}.
\newblock Bi-Mamba4TS: Bidirectional Mamba for Time Series Forecasting. 2024.

\bibitem[Lim et~al.(2020)B.~Lim, S.~O. Arik, N.~Loeff, T.~Pfister]{tft}
{\EM Lim Bryan, Arik Sercan~O., Loeff Nicolas, Pfister Tomas}.
\newblock Temporal Fusion Transformers for Interpretable Multi-horizon Time Series Forecasting. 2020.

\bibitem[Lin et~al.(2023{\natexlab{a}})S.~Lin, W.~Lin, W.~Wu, S.~Wang, Y.~Wang]{petformer}
{\EM Lin Shengsheng, Lin Weiwei, Wu~Wentai, Wang Songbo, Wang Yongxiang}.
\newblock PETformer: Long-term Time Series Forecasting via Placeholder-enhanced Transformer. 2023{\natexlab{a}}.

\bibitem[Lin et~al.(2023{\natexlab{b}})S.~Lin, W.~Lin, W.~Wu, F.~Zhao, R.~Mo, H.~Zhang]{rnn1}
{\EM Lin Shengsheng, Lin Weiwei, Wu~Wentai, Zhao Feiyu, Mo~Ruichao, Zhang Haotong}.
\newblock SegRNN: Segment Recurrent Neural Network for Long-Term Time Series Forecasting. 2023{\natexlab{b}}.

\bibitem[Liu et~al.(2022{\natexlab{a}})M.~Liu, A.~Zeng, M.~Chen, Z.~Xu, Q.~Lai, L.~Ma, Q.~Xu]{cnn1}
{\EM Liu Minhao, Zeng Ailing, Chen Muxi, Xu~Zhijian, Lai Qiuxia, Ma~Lingna, Xu~Qiang}.
\newblock SCINet: Time Series Modeling and Forecasting with Sample Convolution and Interaction. 2022{\natexlab{a}}.

\bibitem[Liu et~al.(2022{\natexlab{b}})S.~Liu, H.~Yu, C.~Liao, J.~Li, W.~Lin, A.~X. Liu, S.~Dustdar]{pyraformer}
{\EM Liu Shizhan, Yu~Hang, Liao Cong, Li~Jianguo, Lin Weiyao, Liu Alex~X., Dustdar Schahram}.
\newblock Pyraformer: Low-Complexity Pyramidal Attention for Long-Range Time Series Modeling and Forecasting \allowbreak\newblock// International Conference on Learning Representations. 2022{\natexlab{b}}.

\bibitem[McAllester(1999)D.~A. McAllester]{generalizationboundmc}
{\EM McAllester David~A.}
\newblock PAC-Bayesian model averaging \allowbreak\newblock// Proceedings of the Twelfth Annual Conference on Computational Learning Theory. New York, NY, USA: Association for Computing Machinery, 1999.  164–170.
\newblock (COLT '99).

\bibitem[Mehtab, Sen(2021)S.~Mehtab, J.~Sen]{Mehtab_2021}
{\EM Mehtab Sidra, Sen Jaydip}.
\newblock Stock Price Prediction Using Convolutional Neural Networks on a Multivariate Time Series. VIII 2021.

\bibitem[Nguyen et~al.(2021)H.~M. Nguyen, P.~J. Turk, A.~D. McWilliams]{nguyen2021forecasting}
{\EM Nguyen Hieu~M, Turk Philip~J, McWilliams Andrew~D}.
\newblock Forecasting COVID-19 hospital census: A multivariate time-series model based on local infection incidence \allowbreak\newblock// JMIR Public Health and Surveillance. 2021. 7, 8. e28195.

\bibitem[Nie et~al.(2023)Y.~Nie, N.~H. Nguyen, P.~Sinthong, J.~Kalagnanam]{patchtst}
{\EM Nie Yuqi, Nguyen Nam~H., Sinthong Phanwadee, Kalagnanam Jayant}.
\newblock A Time Series is Worth 64 Words: Long-term Forecasting with Transformers. 2023.

\bibitem[Qiu et~al.(2020)J.~Qiu, B.~Wang, C.~Zhou]{article1}
{\EM Qiu Jiayu, Wang Bin, Zhou Changjun}.
\newblock Forecasting stock prices with long-short term memory neural network based on attention mechanism \allowbreak\newblock// PLOS ONE. 01 2020. 15. e0227222.

\bibitem[Rasul et~al.(2024)K.~Rasul, A.~Ashok, A.~R. Williams, H.~Ghonia, R.~Bhagwatkar, A.~Khorasani, M.~J.~D. Bayazi, G.~Adamopoulos, R.~Riachi, N.~Hassen, M.~Biloš, S.~Garg, A.~Schneider, N.~Chapados, A.~Drouin, V.~Zantedeschi, Y.~Nevmyvaka, I.~Rish]{rasul2024lagllama}
{\EM Rasul Kashif, Ashok Arjun, Williams Andrew~Robert, Ghonia Hena, Bhagwatkar Rishika, Khorasani Arian, Bayazi Mohammad Javad~Darvishi, Adamopoulos George, Riachi Roland, Hassen Nadhir, Biloš Marin, Garg Sahil, Schneider Anderson, Chapados Nicolas, Drouin Alexandre, Zantedeschi Valentina, Nevmyvaka Yuriy, Rish Irina}.
\newblock Lag-Llama: Towards Foundation Models for Probabilistic Time Series Forecasting. 2024.

\bibitem[Sanhudo et~al.(2021)L.~Sanhudo, J.~Rodrigues, E.~Vasconcelos~Filho]{sanhudo2021multivariate}
{\EM Sanhudo Lu{\'\i}s, Rodrigues Joao, Vasconcelos~Filho Enio}.
\newblock Multivariate time series clustering and forecasting for building energy analysis: Application to weather data quality control \allowbreak\newblock// Journal of Building Engineering. 2021. 35. 101996.

\bibitem[Shao et~al.(2023)Z.~Shao, F.~Wang, Z.~Zhang, Y.~Fang, G.~Jin, Y.~Xu]{hutformer}
{\EM Shao Zezhi, Wang Fei, Zhang Zhao, Fang Yuchen, Jin Guangyin, Xu~Yongjun}.
\newblock HUTFormer: Hierarchical U-Net Transformer for Long-Term Traffic Forecasting. 2023.

\bibitem[Tolstikhin et~al.(2021)I.~Tolstikhin, N.~Houlsby, A.~Kolesnikov, L.~Beyer, X.~Zhai, T.~Unterthiner, J.~Yung, A.~Steiner, D.~Keysers, J.~Uszkoreit, M.~Lucic, A.~Dosovitskiy]{tolstikhin2021mlpmixer}
{\EM Tolstikhin Ilya, Houlsby Neil, Kolesnikov Alexander, Beyer Lucas, Zhai Xiaohua, Unterthiner Thomas, Yung Jessica, Steiner Andreas, Keysers Daniel, Uszkoreit Jakob, Lucic Mario, Dosovitskiy Alexey}.
\newblock MLP-Mixer: An all-MLP Architecture for Vision. 2021.

\bibitem[Valle-Pérez, Louis(2020)G.~Valle-Pérez, A.~A. Louis]{generalizationbound}
{\EM Valle-Pérez Guillermo, Louis Ard~A.}
\newblock Generalization bounds for deep learning. 2020.

\bibitem[Vaswani et~al.(2017)A.~Vaswani, N.~Shazeer, N.~Parmar, J.~Uszkoreit, L.~Jones, A.~N. Gomez, L.~u. Kaiser, I.~Polosukhin]{NIPS2017_3f5ee243}
{\EM Vaswani Ashish, Shazeer Noam, Parmar Niki, Uszkoreit Jakob, Jones Llion, Gomez Aidan~N, Kaiser \L~ukasz, Polosukhin Illia}.
\newblock Attention is All you Need \allowbreak\newblock// Advances in Neural Information Processing Systems.  30. 2017.

\bibitem[Wang et~al.(2023)H.~Wang, J.~Peng, F.~Huang, J.~Wang, J.~Chen, Y.~Xiao]{cnn2}
{\EM Wang Huiqiang, Peng Jian, Huang Feihu, Wang Jince, Chen Junhui, Xiao Yifei}.
\newblock {MICN}: Multi-scale Local and Global Context Modeling for Long-term Series Forecasting \allowbreak\newblock// The Eleventh International Conference on Learning Representations. 2023.

\bibitem[Wang et~al.(2024)S.~Wang, H.~Wu, X.~Shi, T.~Hu, H.~Luo, L.~Ma, J.~Y. Zhang, J.~ZHOU]{timemixer}
{\EM Wang Shiyu, Wu~Haixu, Shi Xiaoming, Hu~Tengge, Luo Huakun, Ma~Lintao, Zhang James~Y., ZHOU JUN}.
\newblock TimeMixer: Decomposable Multiscale Mixing for Time Series Forecasting \allowbreak\newblock// The Twelfth International Conference on Learning Representations. 2024.

\bibitem[Woo et~al.(2023)G.~Woo, C.~Liu, D.~Sahoo, A.~Kumar, S.~Hoi]{deeptime}
{\EM Woo Gerald, Liu Chenghao, Sahoo Doyen, Kumar Akshat, Hoi Steven}.
\newblock Learning Deep Time-index Models for Time Series Forecasting \allowbreak\newblock// Proceedings of the 40th International Conference on Machine Learning. 2023.

\bibitem[Wu et~al.(2023)H.~Wu, T.~Hu, Y.~Liu, H.~Zhou, J.~Wang, M.~Long]{timesnet}
{\EM Wu~Haixu, Hu~Tengge, Liu Yong, Zhou Hang, Wang Jianmin, Long Mingsheng}.
\newblock TimesNet: Temporal 2D-Variation Modeling for General Time Series Analysis. 2023.

\bibitem[Wu et~al.(2022)H.~Wu, J.~Xu, J.~Wang, M.~Long]{autoformer}
{\EM Wu~Haixu, Xu~Jiehui, Wang Jianmin, Long Mingsheng}.
\newblock Autoformer: Decomposition Transformers with Auto-Correlation for Long-Term Series Forecasting. 2022.

\bibitem[Wu et~al.(2020)Z.~Wu, S.~Pan, G.~Long, J.~Jiang, X.~Chang, C.~Zhang]{wu2020connecting}
{\EM Wu~Zonghan, Pan Shirui, Long Guodong, Jiang Jing, Chang Xiaojun, Zhang Chengqi}.
\newblock Connecting the Dots: Multivariate Time Series Forecasting with Graph Neural Networks. 2020.

\bibitem[Xu et~al.(2019)L.~Xu, M.~Skoularidou, A.~Cuesta-Infante, K.~Veeramachaneni]{xu2019modeling}
{\EM Xu~Lei, Skoularidou Maria, Cuesta-Infante Alfredo, Veeramachaneni Kalyan}.
\newblock Modeling Tabular data using Conditional GAN. 2019.

\bibitem[Xu et~al.(2024)Z.~Xu, A.~Zeng, Q.~Xu]{fits}
{\EM Xu~Zhijian, Zeng Ailing, Xu~Qiang}.
\newblock {FITS}: Modeling Time Series with \$10k\$ Parameters \allowbreak\newblock// The Twelfth International Conference on Learning Representations. 2024.

\bibitem[Xue et~al.(2023)W.~Xue, T.~Zhou, Q.~Wen, J.~Gao, B.~Ding, R.~Jin]{card}
{\EM Xue Wang, Zhou Tian, Wen Qingsong, Gao Jinyang, Ding Bolin, Jin Rong}.
\newblock Make Transformer Great Again for Time Series Forecasting: Channel Aligned Robust Dual Transformer. 2023.

\bibitem[Yu et~al.(2023)C.~Yu, F.~Wang, Z.~Shao, T.~Sun, L.~Wu, Y.~Xu]{dsformer}
{\EM Yu~Chengqing, Wang Fei, Shao Zezhi, Sun Tao, Wu~Lin, Xu~Yongjun}.
\newblock DSformer: A Double Sampling Transformer for Multivariate Time Series Long-term Prediction. 2023.

\bibitem[Yun et~al.(2020)C.~Yun, S.~Bhojanapalli, A.~S. Rawat, S.~J. Reddi, S.~Kumar]{yun2020transformers}
{\EM Yun Chulhee, Bhojanapalli Srinadh, Rawat Ankit~Singh, Reddi Sashank~J., Kumar Sanjiv}.
\newblock Are Transformers universal approximators of sequence-to-sequence functions? 2020.

\bibitem[Zeng et~al.(2022)A.~Zeng, M.~Chen, L.~Zhang, Q.~Xu]{dlinear}
{\EM Zeng Ailing, Chen Muxi, Zhang Lei, Xu~Qiang}.
\newblock Are Transformers Effective for Time Series Forecasting? 2022.

\bibitem[Zhang et~al.(2023{\natexlab{a}})Y.~Zhang, Q.~Wen, X.~Wang, W.~Chen, L.~Sun, Z.~Zhang, L.~Wang, R.~Jin, T.~Tan]{onenet}
{\EM Zhang YiFan, Wen Qingsong, Wang Xue, Chen Weiqi, Sun Liang, Zhang Zhang, Wang Liang, Jin Rong, Tan Tieniu}.
\newblock OneNet: Enhancing Time Series Forecasting Models under Concept Drift by Online Ensembling \allowbreak\newblock// Thirty-seventh Conference on Neural Information Processing Systems. 2023{\natexlab{a}}.

\bibitem[Zhang, Yan(2023)Y.~Zhang, J.~Yan]{crossformer}
{\EM Zhang Yunhao, Yan Junchi}.
\newblock Crossformer: Transformer Utilizing Cross-Dimension Dependency for Multivariate Time Series Forecasting \allowbreak\newblock// The Eleventh International Conference on Learning Representations. 2023.

\bibitem[Zhang et~al.(2023{\natexlab{b}})Z.~Zhang, X.~Wang, Y.~Gu]{sageformer}
{\EM Zhang Zhenwei, Wang Xin, Gu~Yuantao}.
\newblock SageFormer: Series-Aware Graph-Enhanced Transformers for Multivariate Time Series Forecasting. 2023{\natexlab{b}}.

\bibitem[Zhao et~al.(2023)Y.~Zhao, Z.~Ma, T.~Zhou, L.~Sun, M.~Ye, Y.~Qian]{gcformer}
{\EM Zhao YanJun, Ma~Ziqing, Zhou Tian, Sun Liang, Ye~Mengni, Qian Yi}.
\newblock GCformer: An Efficient Framework for Accurate and Scalable Long-Term Multivariate Time Series Forecasting. 2023.

\bibitem[Zhou et~al.(2021)H.~Zhou, S.~Zhang, J.~Peng, S.~Zhang, J.~Li, H.~Xiong, W.~Zhang]{informer}
{\EM Zhou Haoyi, Zhang Shanghang, Peng Jieqi, Zhang Shuai, Li~Jianxin, Xiong Hui, Zhang Wancai}.
\newblock Informer: Beyond Efficient Transformer for Long Sequence Time-Series Forecasting. 2021.

\bibitem[Zhou et~al.(2022)T.~Zhou, Z.~Ma, Q.~Wen, X.~Wang, L.~Sun, R.~Jin]{fedformer}
{\EM Zhou Tian, Ma~Ziqing, Wen Qingsong, Wang Xue, Sun Liang, Jin Rong}.
\newblock FEDformer: Frequency Enhanced Decomposed Transformer for Long-term Series Forecasting. 2022.

\bibitem[donghao, xue wang(2024)L.~donghao, wang xue]{moderntcn}
{\EM donghao Luo, xue  wang}.
\newblock Modern{TCN}: A Modern Pure Convolution Structure for General Time Series Analysis \allowbreak\newblock// The Twelfth International Conference on Learning Representations. 2024.

\end{thebibliography}

\clearpage



\appendix

\section{Proof}\label{appen:proof}

\subsection{Proof for Theorem~\ref{theorm:generalizationbound}}
Starting from McAllester’s bound on generalization errors~\citep{generalizationboundmc}, we derive generalization bound in Theorem~\ref{theorm:generalizationbound}. Before getting into the main part, we define some notations. Let a neural network $f$ be a partial-multivariate model which samples subsets $\mathbf{F}$ consisting of $S$ features from a complete set of $D$ features as defined in equation~\eqref{eq:partial_multivariate_model}. $\mathcal{H}$ denotes hypothesis class of $f$, and $\mathbf{P}(h)$ and $\mathbf{Q}(h)$ are a prior and posterior distribution over the hypotheses $h$, respectively. Also, $(\mathbf{x},\mathbf{y})$ is a input-output pair in an entire dataset and $(\mathbf{x}^{\mathcal{T}},\mathbf{y}^{\mathcal{T}})$ is a pair in a training dataset $\mathcal{T}$ with $m$ instances sampled from the entire dataset. At last, $\hat{\mathbf{y}}=f(\mathbf{x})$ is the output value of a neural network $f$, and $l(\mathbf{Q})$ and $\hat{l}(\mathbf{Q}, \mathcal{T})$ are generalized and empirical training loss under posterior distributions $\mathbf{Q}$ and training datasets $\mathcal{T}$.

Subsequently, we list assumptions for proof:
\begin{assumption}
 The maximum and minimum values of $\mathbf{y}$ are known and min-max normalization is applied to $\mathbf{y}$ (\textit{i.e.}, $0 \le \mathbf{y} \le 1$). 
\label{assumption:minmaxnorm}
\end{assumption}
\begin{assumption}
The output values of a neural network are assumed to be between 0 and 1, (\textit{i.e.}, $0\le \hat{\mathbf{y}}\le 1$). 
\label{assumption:minmaxval}
\end{assumption}
\begin{assumption}
For posterior distributions $\mathbf{Q}$, $\mathbf{Q}$ is pruned. In other words, we set $\mathbf{Q}(h)=0$ for hypotheses $h$ where $\mathbf{Q}(h) < \mathbf{P}(h)$ and renormalize it.
\label{assumption:posterior_prune}
\end{assumption}
\begin{assumption}
For any hypothesis $h$, $\mathbf{P}(h) > \omega$ where $\omega$ is the minimum probabilities in $\mathbf{P}(h)$ and $\omega>0$.
\label{assumption:minprob}
\end{assumption}
\begin{assumption}
For posterior distributions $\mathbf{Q}$ and training datasets $\mathcal{T}$, $\hat{l}(\mathbf{Q},\mathcal{T}) \approx 0$.
\label{assumption:posterior_approx_0}
\end{assumption}

Given that min-max normalization has been often used in time-series domains with empirical minimum and maximum values~\citep{bhanja2019impact}, Assumption~\ref{assumption:minmaxnorm} can be regarded as a reasonable one. Also, by equipping the last layer with some activation functions such as Sigmoid or Tanh (hyperbolic tangent) like~\citet{xu2019modeling} and adequate post-processing, Assumption~\ref{assumption:minmaxval} can be satisfied.\footnote{Assumption~\ref{assumption:minmaxnorm} and~\ref{assumption:minmaxval} can be considered somewhat strong but should be satisfied to utilize McAllester’s bound widely used for estimating generalization errors~\citep{generalizationbound,generalizationbound2}. When the conditions of McAllester’s bound are relaxed, we can also relax our assumptions.} As for Assumption~\ref{assumption:posterior_prune}, according to~\citep{generalizationboundmc}, it might have very little effects on $\mathbf{Q}$. Finally, because Transformers can universally approximate any continuous sequence-to-sequence function~\citep{yun2020transformers}, (possibly, extended to general deep neural networks with the universal approximation theorem~\citep{Cybenko1989}), any hypothesis $h$ can be approximated with proper parameters in $f$. Thus, we can assume $\mathbf{P}(h) > w > 0$ for any $h$ when sampling the initial parameters of $f$ from the whole real-number space (Assmuption~\ref{assumption:minprob}). Also with proper training process and this universal approximation theorem, $\hat{l}(\mathbf{Q},\mathcal{T})$ might approximate to zero (Assumption~\ref{assumption:posterior_approx_0}). With these assumptions, the proof for Theorem~\ref{theorm:generalizationbound} is as follows:

\begin{proof}
Let MSE be a loss function $l$. Then, according to Assumption~\ref{assumption:minmaxnorm} and~\ref{assumption:minmaxval}, $0\le l(h,(\mathbf{x},\mathbf{y}))\le 1$ for any data instance $(\mathbf{x},\mathbf{y})$ and hypothesis $h$. Then, with probability at least $1 - \delta$ over the selection of the sample $\mathcal{T}$ of size $m$, we have the following for $\mathbf{Q}$~\citep{generalizationboundmc}:
\begin{align}
    l(\mathbf{Q}) \le \hat{l}(\mathbf{Q},\mathcal{T}) + \sqrt{\frac{D(\mathbf{Q}\|\mathbf{P}) + \log \frac{1}{\delta} + \frac{5}{2} \log m + 8}{2m - 1}},
\end{align}
where $D(\mathbf{Q}\|\mathbf{P})$ denotes Kullback-Leibler divergence from distribution $\mathbf{Q}$ to $\mathbf{P}$. Due to Assumption~\ref{assumption:posterior_approx_0}, $\hat{l}(\mathbf{Q},\mathcal{T}) \approx 0$. Also, because $E[\log \frac{1}{\mathbf{P}(h)}] < \log E[\frac{1}{\mathbf{P}(h)}] < \log \frac{1}{\omega} = C$ with Jensen's inequality and Assumption~\ref{assumption:minprob}, $D(\mathbf{Q}\|\mathbf{P})=E_{h\sim \mathbf{Q}}[\log \frac{\mathbf{Q}(h)}{\mathbf{P}(h)}] = E[\log \mathbf{Q}(h)] + E[\log \frac{1}{\mathbf{P}(h)}] < E[\log \mathbf{Q}(h)] + C$. Therefore, we can derive Theorem~\ref{theorm:generalizationbound} by substituting $\hat{l}(\mathbf{Q},\mathcal{T})$ and $D(\mathbf{Q}\|\mathbf{P})$ with $0$ and $E[\log \mathbf{Q}(h)] + C$, respectively:
\begin{align}
    l(\mathbf{Q}) \le \sqrt{\frac{E_{h\sim \mathbf{Q}}[\log \mathbf{Q}(h)] + \log \frac{1}{\delta} + \frac{5}{2} \log m + 8  + C}{2m - 1}}.
\end{align}

\end{proof}

Based on this theorem, we provide a theoretical analysis which is the impact of $S$ on $m$ and $-H(Q)$. However, an additional assumption is required to make the rationale valid as follows:
\begin{assumption}
    For the region of hypothesis $h'$ where $\mathbf{Q}(h') > 0$, the prior distribution satisfies $\log \frac{1}{\mathbf{P}(h')} \le C_{max}$ where $C_{max}$ is small enough to be ignored in upper-bound.
    \label{assumption:cmax}
\end{assumption}

It is possible that the upper-bound is dominated by $C\rightarrow \infty$ when $w \rightarrow 0$. As such, $P(h)$ needs to be distributed properly over the region of hypothesis $h'$ where $\mathbf{Q}(h') > 0$ not to result in $C\rightarrow \infty$, leading to Assumption~\ref{assumption:cmax}. This assumption can be satisfied when the prior distribution is non-informative which is natural in Bayesian statistics under the assumption that prior knowledge is unknown (i.e. $P(h) \propto 1$). For any countable set of all possible inputs $\{\mathbf{x}_i\}_{i=1}^{N}$, probabilities of each $h$ can be represented as $p(h)=\prod_{i=1}^{N} p(\hat{\mathbf{y}}^h_i|\mathbf{x}_i)$ where $\hat{\mathbf{y}}^h_i = f_h(\mathbf{x}_i)$ is the output of a function $f_h$ under hypothesis $h$~\citep{10.1145/2347736.2347755}. Because $0 \le \hat{\mathbf{y}}^h_i \le 1$ (Assumption~\ref{assumption:minmaxval}) and  $p(\hat{\mathbf{y}}^h_i|\mathbf{x}_i)$ is a uniform distribution under the non-informative assumption, $p(\hat{\mathbf{y}}^h_i|\mathbf{x}_i)=1$. As such, the prior distribution under the non-informative assumption is $\mathbf{P}(h)=1$, leading to $C_{max}=0$ which is small enough not to dominate upper-bound. On top of that, we can indirectly solve this problem by injecting appropriate inductive biases in the form of architectures or regularizers, which can help to allocate more probability to each hypothesis (\textit{i.e.}, increase $\omega$) by reducing the size of the whole hypothesis space $\mathcal{H}$.

\subsection{Proof for Theorem~\ref{theorem:hypothesisspace}}

To provide a proof for Theorem~\ref{theorem:hypothesisspace}, we first prove Lemma~\ref{lemma:s+-}. For Lemma~\ref{lemma:s+-}, we need the following assumption:
\begin{assumption}
A neural network $f$ models models $p(\mathbf{y}|\mathbf{x})$ where $(\mathbf{x},\mathbf{y})$ is an input-output pair.
\label{assumption:probmodel}
\end{assumption}
By regarding the output of a neural network $\hat{\mathbf{y}}$ as mean of normal or Student's $t$-distribution like in~\citet{rasul2024lagllama}, Assumption~\ref{assumption:probmodel} can be satisfied. Then, Lemma~\ref{lemma:s+-} and a proof are as follows: 
\begin{lemma}
    Let $\hat{l}(\mathbf{Q}_S,\mathcal{T}_S)$ be a training loss with posterior distributions $\mathbf{Q}_S$ and a training dataset $\mathcal{T}_S$ when a subset size is $S$. Accordingly, $\hat{l}(\mathbf{Q}_S,\mathcal{T}_S) < \epsilon $ with small $\epsilon$ is a training objective. Then, for $S_{+}$ and $S_{-}$ where $S_{+} > S_{-}$, $\mathbf{Q}_{S_{+}}$ satisfies both $\hat{l}(\mathbf{Q}_{S_{+}},\mathcal{T}_{S_{+}}) < \epsilon $ and $\hat{l}(\mathbf{Q}_{S_{+}},\mathcal{T}_{S_{-}}) < \epsilon $. (On the other hands, $\mathbf{Q}_{S_{-}}$ is required to satisfy only $\hat{l}(\mathbf{Q}_{S_{-}},\mathcal{T}_{S_{-}}) < \epsilon $.)
    \label{lemma:s+-}
\end{lemma}
 
\begin{proof}
    Let $S_+$ and $S_-$ be subset size where $S_+ > S_-$. $\mathbf{F}_{S_-}$ be any subset of $S_-$ size sampled from a complete set of features, and $\mathbf{F}_{S_+}$ is any subset of $S_+$ size among ones that satisfy $\mathbf{F}_{S_-} \subset \mathbf{F}_{S_+} $. $\mathbf{F}_R$ is the set of elements that are in $\mathbf{F}_{S_+}$ but not in $\mathbf{F}_{S_-}$ (\textit{i.e.}, $\mathbf{F}_R = \mathbf{F}_{S_+} - \mathbf{F}_{S_-}$). 
    $\hat{l}(\mathbf{Q}_S,\mathcal{T}_S)$ is a training loss value with posterior distributions $\mathbf{Q}_S$ and a training dataset $\mathcal{T}_S$ when a subset size is $S$. Then, after training process satisfying $\hat{l}(\mathbf{Q}_{S_+},\mathcal{T}_{S_+}) < \epsilon$ where $\epsilon$ is a small value, we can say that $f$ under $\mathbf{Q}_{S_+}$ outputs the true value of $p(\mathbf{y}_{\mathbf{F}_{S_+}}|\mathbf{x}_{\mathbf{F}_{S_+}})$, according to Assumption~\ref{assumption:probmodel}. In the following process, we demonstrate that $p(\mathbf{y}_{\mathbf{F}_{S_-}}|\mathbf{x}_{\mathbf{F}_{S_-}})$ can be derived from $p(\mathbf{y}_{\mathbf{F}_{S_+}}|\mathbf{x}_{\mathbf{F}_{S_+}}) = p(\mathbf{y}_{\mathbf{F}_{S_-}}, \mathbf{y}_{\mathbf{F}_R} | \mathbf{x}_{\mathbf{F}_{S_-}}, \mathbf{x}_{\mathbf{F}_R})$:

    \begin{align}
        & \int_{\mathbf{y}_{\mathbf{F}_R}} E_{\mathbf{x}_{\mathbf{F}_R}|\mathbf{x}_{\mathbf{F}_{S_-}}}[p(\mathbf{y}_{\mathbf{F}_{S_-}}, \mathbf{y}_{\mathbf{F}_R} | \mathbf{x}_{\mathbf{F}_{S_-}}, \mathbf{x}_{\mathbf{F}_R})] d\mathbf{y}_{\mathbf{F}_R}, \\ 
        =\quad  & \int_{\mathbf{y}_{\mathbf{F}_R}} \int_{\mathbf{x}_{\mathbf{F}_R}} p(\mathbf{y}_{\mathbf{F}_{S_-}}, \mathbf{y}_{\mathbf{F}_R} | \mathbf{x}_{\mathbf{F}_{S_-}}, \mathbf{x}_{\mathbf{F}_R}) p(\mathbf{x}_{\mathbf{F}_R}|\mathbf{x}_{\mathbf{F}_{S_-}}) d\mathbf{x}_{\mathbf{F}_R} d\mathbf{y}_{\mathbf{F}_R}, \\ 
        = \quad & p(\mathbf{y}_{\mathbf{F}_{S_-}}|\mathbf{x}_{\mathbf{F}_{S_-}}),
    \end{align}

    In that expectation can be approximated by an empirical mean with sufficient data and integral can be addressed with discretization, we can think that $p(\mathbf{y}_{\mathbf{F}_{S_-}}|\mathbf{x}_{\mathbf{F}_{S_-}})$ can be derived from $p(\mathbf{y}_{\mathbf{F}_{S_+}}|\mathbf{x}_{\mathbf{F}_{S_+}})$. According to this fact, $f$ under $\mathbf{Q}_+$ should be able to output not only true $p(\mathbf{y}_{\mathbf{F}_{S_+}}|\mathbf{x}_{\mathbf{F}_{S_+}})$ but also true $p(\mathbf{y}_{\mathbf{F}_{S_-}}|\mathbf{x}_{\mathbf{F}_{S_-}})$. Therefore, we conclude that $\mathbf{Q}_+$ have to satisfy both $\hat{l}(\mathbf{Q}_{S_{+}},\mathcal{T}_{S_{+}}) < \epsilon $ and $\hat{l}(\mathbf{Q}_{S_{+}},\mathcal{T}_{S_{-}}) < \epsilon $.
\end{proof}


With Lemma~\ref{lemma:s+-}, we provide a proof for Theorem~\ref{theorem:hypothesisspace}:
\begin{proof}
Let $h$ be a hypothesis on a space defined when a subset size is $S$. Then, we can denote a posterior distribution which is trained to decrease $\hat{l}(\mathbf{Q}_S,\mathcal{T}_{S})$ as follows:
\begin{align}
    \mathbf{Q}(h_S)=p(h_S|c_S=1), \quad \text{where} \; \; c_S =     \begin{cases}
    1, & \hat{l}(h,\mathcal{T}_{S}) < \epsilon, \\
    0,              & \text{otherwise},
\end{cases}
\end{align}
According to Lemma~\ref{lemma:s+-}, for $S_+$ and $S_-$ where $S_+ > S_-$, the posterior distributions of two cases can be represent as $\mathbf{Q}(h_{S_+})=p(h_{S_+}|c_{S_+}=1, c_{S_-}=1)$ and $\mathbf{Q}(h_{S_-})=p(h_{S_-}|c_{S_-}=1)$, respectively. With the following two assumptions, we can prove Theorem~\ref{theorem:hypothesisspace}:
\begin{assumption}
hypotheses $h_{S_+}$ and $h_{S_-}$ have similar distributions after training with $\mathcal{T}_{S_-}$ (\textit{i.e.,} $p(h_{S_+}|c_{S_-}=1) \approx p(h_{S-}|c_{S_-}=1)$).
\label{assumption:similar_entropy}
\end{assumption}
\begin{assumption}
Prior distributions are nearly non-informative (i.e., $P(h) \propto 1$).
\label{assumption:non_inform}
\end{assumption}
Assumption~\ref{assumption:similar_entropy} can be considered reasonable because we can make the training process of a model of subset size $S_+$ very similar to that of subset size $S_-$ with a minimal change in architecture such as input and output masking. 
Also, as for Assumption~\ref{assumption:non_inform}, non-informative prior is usually used under usual situations without prior knowledge in Bayesian statistics.

$\mathbf{Q}(h_{S})$ can be expanded as $p(h_{S}|c_{S}) \propto p(c_{S}|h_{S})p(h_{S}) \propto p(c_{S}|h_{S})$, according to Assumption~\ref{assumption:non_inform}. Because we exactly know whether to satisfy $\hat{l}(h,\mathcal{T}_{S})< \epsilon$ given $h$, $p(c_S|h_S)$ is 1 when a given $h_S$ satisfies $c_S$ or 0, otherwise. Thus, $\mathbf{Q}(h_{S})$ are defined as follows:
\begin{align}
    \mathbf{Q}(h_{S})= p(h_{S}|c_{S}=1) = 
    \begin{cases}
    \eta_S, & c_S = 1 \text{ given } h_S, \\
    0,      & \text{otherwise},
    \end{cases}
\end{align}

Similarly, $\mathbf{Q}(h_{S_+})$ and $\mathbf{Q}(h_{S_-})$ can be expanded as $p(h_{S_+}|c_{S_+},c_{S_-}) \propto p(c_{S_+},c_{S_-}|h_{S_+})p(h_{S_+})\propto p(c_{S_+},c_{S_-}|h_{S_+})$ and $p(h_{S_-}|c_{S_-})=p(h_{S_+}|c_{S_-}) \propto p(c_{S_-}|h_{S_+})p(h_{S_+})\propto p(c_{S_-}|h_{S_+})$, according to Assumption~\ref{assumption:similar_entropy} and~\ref{assumption:non_inform}. A region of hypothesis satisfying both $c_{S_+}=1$ and $c_{S_-}=1$ is smaller than that satisfying either of them. Because the probability of $h$ in a region satisfying conditions has the same value and $\int_h p(h)dh=1$ is maintained, $h$ in the small region is allocated higher probabilities than $h$ in the large one. Therefore, $\eta_{S_+} > \eta_{S_-}$ and the entropy $H(\mathbf{Q}_{S_-})$ is larger than $H(\mathbf{Q}_{S_+})$:



\end{proof}

So far, we have finished a proof for Theorem~\ref{theorem:hypothesisspace}. We additionally provide Theorem~\ref{theorem:hypothesisspace2} which is a variant of Theorem~\ref{theorem:hypothesisspace} where Assumption~\ref{assumption:non_inform} can be relaxed while  proposing the relationships between $H(\mathbf{Q}_{S_-})$ and $H(\mathbf{Q}_{S_+})$ in the expectation level:
\begin{theorem}
    for $S_{+}$ and $S_{-}$ satisfying $S_{+} > S_{-}$, $H(\mathbf{Q}_{S_+}) \le H(\mathbf{Q}_{S_-})$ in expectation over $c_{S_+}$. 
    \label{theorem:hypothesisspace2}
\end{theorem}
\begin{proof}
    
Let $\tilde{h}_{S_+}$ be the $h_{S_+}|c_{S_-}=1$ (\textit{i.e.}, $\mathbf{Q}(h_{S_+})=p(h_{S_+}|c_{S_+}=1, c_{S_-}=1)=p(\tilde{h}_{S_+}|c_{S_+}=1)$). 
Then, $H(\tilde{h}_{S_+}|c_{S_+})$ can be expanded as follows:
\begin{align}
& H(p(\tilde{h}_{S_+}|c_{S_+})), \\
\begin{split}
 &= H(p(c_{S_+}|\tilde{h}_{S_+})) + H(p(\tilde{h}_{S_+})) - H(p(c_{S_+})), \\
&(\because \text{Bayes' rule for conditional entropy states}), 
\end{split}\\ 
\begin{split}
&= H(p(\tilde{h}_{S_+})) - H(p(c_{S_+})) \\
&(\because \text{when $h$ is given, we know whether to satisfy } \hat{l}(h,\mathcal{T}_{S_+}) < \epsilon.\ (i.e., H(p(c_{S_+}|\tilde{h}_{S_+}))=0),
\end{split}
\end{align}
From this expansion, we can derive $H(p(\tilde{h}_{S_+}|c_{S_+})) \le H(p(\tilde{h}_{S_+}))$ because entropy of $p(c_{S_+})$ must be larger than 0 (\textit{i.e.,} $H(p(c_{S_+})) \ge 0$). By substituting $H(\mathbf{Q}_{S_-})$ for $H(p(\tilde{h}_{S_+}))$ according to Assumption~\ref{assumption:similar_entropy} and $E_{c_{S_+}}[H(\mathbf{Q}_{S_+})]$ for $H(p(\tilde{h}_{S_+}|c_{S_+}))$, we can derive Theorem~\ref{theorem:hypothesisspace2}.

Also, based on Chebyshev's inequality, we can calculate the least probabilities at which $H(\mathbf{Q}_{S_+}) < H(\mathbf{Q}_{S_-})$ are satisfied, given the variance $\sigma^2 = Var_{c_{S_+}}[H(p(\tilde{h}_{S_+}|c_{S_+}))]$:
\begin{align}
    p\left[H(\mathbf{Q}_{S_+}) < H(\mathbf{Q}_{S_-}) \right] \quad \le \quad 1 - \frac{\sigma^2}{(H(p(c_{S_+})))^2}
\end{align}
\end{proof}

\section{How to Handle Non-Divisible Cases of \method with Random Partitioning}\label{appen:non-div}
In this section, we further elaborate on how to deal with the cases where the number of features $D$ is not divisible by the size of subsets $S$.  We simply repeat some randomly chosen features and augment them to the original input time series, in order to make the total number of features divisible by $S$. After finishing the forecasting procedure with the augmented inputs, we drop augmented features from outputs. The details are delineated in Algorithm~\ref{alg:non_div}.

\begin{algorithm}[ht]
\small
\caption{How to handle non-divisible cases of \method with random partitioning}\label{alg:non_div}
\KwIn{\# of features $D$, Subset size $S$, Past obs.  $\mathbf{x}_{[0:D]}$}

$\mathbf{V} = \{0,1,...,D-1\}$; $\quad $ $N_U = \lceil \frac{D}{S} \rceil$; $\quad $ $R = D\; \% \; S$;
   
\eIf{$R\neq 0$} {
     Randomly split $\mathbf{V}$ into $\mathbf{V}^{+},\mathbf{V}^{-}$, where $|\mathbf{V}^{+}|=D - R, |\mathbf{V}^{-}|=R, \mathbf{V}^+ \cap \mathbf{V}^- = \phi$;
     
     Get $\{\mathbf{F}(g)\}_{g \in [0,N_U-1]}$ by randomly partitioning $\mathbf{V}^+$;

    $\mathbf{V}^{++} = \{v_i|v_i \text{is a random sample from}\ \mathbf{V}^{+}\ \text{without replacement}, i = [0,S-R]\}$;
    
    $\mathbf{F}(N_U - 1) = \mathbf{V}^{-} \cup \mathbf{V}^{++}$
}{
     Get $\{\mathbf{F}(g)\}_{g \in [0,N_U]}$ by randomly partitioning $\mathbf{V}$;
}

\For{$g \gets 0$  \KwTo $N_U-1$}{
        $\hat{\mathbf{y}}_{\mathbf{F}(g)} = \mathtt{PMformer}
    (\mathbf{x}_{\mathbf{F}(g)}, \mathbf{F}(g))$; 
}
\If{$R\neq 0$}{
    Remove features of $\mathbf{V}^{++}$ from $\hat{\mathbf{y}}_{\mathbf{F}(N_U-1)}$;
}
    Sort $\{\hat{\mathbf{y}}_{\mathbf{F}(g)}\}_{g\in[0,N_U]}$ by feature index and get $\hat{\mathbf{y}}_{[0:D]}$;

\Return Predicted future observations $\hat{\mathbf{y}}_{[0:D]}$;
\end{algorithm}

\section{Details of Experimental Environments}\label{appen:detail_exp}

\subsection{Datasets}
We evaluate \method on 7 benchmark datasets for time series forecasting with multiple variables. The normalization and train/val/test splits are also the same with PatchTST~\citep{patchtst} which is our main baseline. The information of each dataset is as follows:
\begin{itemize}
[noitemsep,topsep=0pt,parsep=3pt,partopsep=1pt,leftmargin=10pt]
    \item \textbf{(1-2)} \textbf{ETTh1},\textbf{2}\footnote{\url{https://github.com/zhouhaoyi/ETDataset}} (Electricity Transformer Temperature-hourly): They have 7 indicators in the electric power long-term deployment, such as oil temperature and 6 power load features. This data is collected for 2 years and the granularity is 1 hour. Different numbers denote different counties in China. Train/val/test is 12/4/4 months and the number of time steps is 17,420.

    \item \textbf{(3-4)} \textbf{ETTm1},\textbf{2} (Electricity Transformer Temperature-minutely): This dataset is exactly the same with \textbf{ETTh1},\textbf{2}, except for granularity. The granularity of these cases is 15 minutes. The number of time steps is 69,680.
    
    \item \textbf{(5)} \textbf{Weather}\footnote{\url{https://www.bgc-jena.mpg.de/wetter/}}: It has 21 indicators of weather including temperature, humidity, precipitation, and air pressure. It was recorded for 2020, and the granularity is 10 minutes. The ratio of train/val/test is 0.7/0.1/0.2 and the number of time steps is 52,696.

    \item \textbf{(6)} \textbf{Electricity}\footnote{\url{https:// archive.ics.uci.edu/ml/datasets/ElectricityLoadDiagrams20112014}}: In this dataset, information about hourly energy consumption from 2012 to 2014 is collected. Each feature means the electricity consumption of one client, and there are 321 clients in total. The ratio of train/val/test is 0.7/0.1/0.2 and the number of time steps is 26,304.
    
    \item \textbf{(7)} \textbf{Traffic}\footnote{\url{http://pems.dot.ca.gov}}:  Traffic dataset pertains to road occupancy rates. It encompasses hourly data collected by 862 sensors deployed on San Francisco freeways during the period spanning from 2015 to 2016. The ratio of train/val/test is 0.7/0.1/0.2 and the number of time steps is 17,544.

\end{itemize}
\subsection{Software \& Hardware Environments and Computation Time}\label{appen:soft_hard_comptime}
We conduct experiments on this software and hardware environments for M-LTSF: \textsc{Python} 3.7.12, \textsc{PyTorch} 2.0.1, and \textsc{NVIDIA GeForce RTX 3090}. For each training of \method, it takes about 1 $\sim$ 4 hours with one or two GPUs according to the number of features. In total, it takes about one months to complete our projects with 16 GPUs.

\subsection{Baselines}\label{appen:baseline}
We select 3 univariate baselines: PatchTST, NLinear, and DLinear. As for complete-multivariate ones, there are many candidates including ~\citet{tft,autoformer,logtrans}. Among them, our choices are Crossformer, FEDformer, Informer, Pyraformer, TSMixer, DeepTime, MICN, and TimesNet, considering their performance and meanings in the forecasting tasks.

\begin{itemize}
[noitemsep,topsep=0pt,parsep=3pt,partopsep=1pt,leftmargin=10pt]

    \item \textbf{(1)} \textbf{PatchTST}~\citep{patchtst}: It uses segmentation with separate tokenization where different features are allocated to different tokens and doesn't consider any relationship between different features.
    \item \textbf{(2)} \textbf{DLinear}~\citep{dlinear}: A single linear layer mapping past observations into future observations with a decomposition trick.
    \item \textbf{(3)} \textbf{NLinear}~\citep{dlinear}: A single linear layer mapping past observations into future observations with a normalization trick that subtracts the last value of input observations from input and adds the value to the output.
    \item \textbf{(4)} \textbf{Crossformer}~\citep{crossformer}: It use similar segmentation to \textbf{PatchTST} and and two types of attention, one of which is self-attention for temporal dependencies and the other is for inter-feature relationships. It reduces the complexity of self-attention for inter-feature relationships using routers with low-rank approximation concepts.
    \item \textbf{(5)} \textbf{FEDformer}~\citep{fedformer}: Using the sparsity in frequency domains, it tries to reduce the quadratic complexity of self-attention layers to a linear one.
    \item \textbf{(6)} \textbf{Informer}~\citep{informer}: By estimating KL divergence between query-key distribution and uniform distribution, it discerns useful and useless information. By using only useful information, it achieves log-linear complexity. Also, a new type of decoder was proposed, which generates forecasting outputs at once.
    \item \textbf{(7)} \textbf{Pyraformer}~\citep{pyraformer}: It has hierarchical structures with different resolutions, leading to linear complexity of self-attention.
    \item \textbf{(8)} \textbf{TSMixer}~\citep{tsmixer}: Using the concept of MLP-Mixer in vision domains~\citep{tolstikhin2021mlpmixer}, it was devised to explore the abilities of linear layers in the forecasting tasks.
    \item \textbf{(9)} \textbf{DeepTime}~\citep{deeptime}: It solves the problem where INRs are hard to be generalized in time-series forecasting tasks, with a meta-optimization framework.
    \item \textbf{(10)} \textbf{MICN}~\citep{cnn2}: To capture both local and global patterns from time series efficiently, it extracts patterns with down-sampled convolution and isometric convolution. Also, multi-scale structures are used to capture more diverse patterns.
    \item \textbf{(11)} \textbf{TimesNet}~\citep{timesnet}: Building upon the multi-periodicity of time series, it regards time series as not 1d but 2d structures and aims to figure out intra-period and inter-period relationships.
\end{itemize}

Furthermore, we find concurrent works for time-series forecasting and include them as our baselines. Among~\citet{jtft,gcformer,card,client,sageformer,hutformer,dsformer,petformer,pits,moderntcn,fits,timemixer}, we select PITS, FITS, TimeMixer, JTFT, GCformer, CARD, Client, PETformer, and ModernTCN as our baselines because they have the same experimental settings with ours\footnote{We decide that it has the same experimental setting with ours when the scores of some baselines are the same.} or their executable codes are available to run models in our settings.
\begin{itemize}
[noitemsep,topsep=0pt,parsep=3pt,partopsep=1pt,leftmargin=10pt]
    \item \textbf{(1)} \textbf{PITS}~\citep{pits}: This paper proposed new a self-supervised representation learning strategy for time series with neural networks not directly considering patch-dependence. Through the contrastive learning, adjacent time series information can be captured efficiently.
    \item \textbf{(2)} \textbf{FITS}~\citep{fits}: This paper introduced FITS, which is effective but efficient for time series tasks, based on the fact that time series can be dealt with in the complex frequency domain. 
    \item \textbf{(3)} \textbf{TimeMixer}~\citep{timemixer}: This paper is similar to \textbf{TSMixer}. However, it has distinct differences in that it utilizes a decomposition scheme and considers multi-scale time series.
    \item \textbf{(4)} \textbf{JTFT}~\citep{jtft}: Similar to \textbf{Crossformer}, segmentation and two types of Transformers are employed. Before a Transformer takes input, it pre-processes input time series. It only encodes a fixed length of recent observations into tokens and sparse frequency information extracted from the whole input into tokens, rather than encodes the whole input directly. This leads to efficient self-attention for temporal dependencies. Also, with a low-rank approximation scheme, it reduces the complexity of self-attention for inter-feature dependencies.
    \item \textbf{(5)} \textbf{GCformer}~\citep{gcformer}: To overcome the limitations of Transformers that they cannot deal with long time series well, it combines a convolutional branch for global information and Transformer-based branch for local, recent information.
    \item \textbf{(6)} \textbf{CARD}~\citep{card}: With a dual Transformer, it can capture various dependencies across  temporal, feature, and hidden dimensions. On top of that, the author devised a robust loss function to relieve overfitting issues in M-LTSF.
    \item \textbf{(7)} \textbf{Client}~\citep{client}: This method has two parts, one of which is a linear model to capture temporal trends and the other is self-attention for inter-feature dependencies.
    \item \textbf{(8)} \textbf{PETformer}~\citep{petformer}: Based on \textbf{Crossformer} architecture, it introduced placeholder enhancement technique (PET). Thanks to PET, PETformer can forecast with only encoders (\textit{i.e.}, without decoder).
    \item \textbf{(9)} \textbf{ModernTCN}~\citep{moderntcn}: Because many existing CNN-based methods don't show the good performance in time series forecasting tasks, this paper tried to modify the traditional CNNs into ModernTCN including maintaining the variable dimension, DWConv, and ConvFFN.
\end{itemize}

As for evaluation metrics of baseline methods, we repeat the scores when the scores of the same experimental settings as ours are available. Otherwise, we measure evaluation scores with their official codes and best hyperparameters in our experimental environments. The scores of PatchTST, FEDformer, Pyraformer, and Informer are from~\citet{patchtst}, and those of TSMixer and NLinear (DLinear) are from~\citet{tsmixer} and~\citet{dlinear}, respectively. Also, for PITS, FITS, TimeMixer, JTFT, GCformer, CARD, PETformer, and ModernTCN, we repeat the score reported in each paper. For Crossformer\footnote{\url{https://github.com/Thinklab-SJTU/Crossformer}}, MICN, TimesNet, Client\footnote{For MICN, TimesNet, and Client, we use the same code from \url{https://github.com/daxin007/Client/tree/main}.}, and DeepTime\footnote{\url{https://github.com/salesforce/DeepTime}}, we measure new scores in the same experimental environments with ours. 
When training Crossformer, we convert a Transformer-based encoder into a linear-based encoder for fair comparison to \method, because the latter usually has better performance than the former.

\subsection{HyperParameters}
The details of hyperparameters used in the \method are delineated in this section. The first hyperparameter is the length of input time steps $T$. We regard it as hyperparameters which is common in recent literature for time-series forecasting~\citep{pyraformer, crossformer}. The range of $T$ is \{512, 1024\}. Also, the number of segments $N_S$ is in \{8,16,32,64\} and the dropout ratio $r_{\text{dropout}}$ is in \{0.1, 0.2, 0.3, 0.4, 0.7\}. The hidden dimension $d_h$ is in \{32,64,128,256,512\}.  The number of heads in self-attention $n_h$ is in \{2,4,8,16\} and the number of layers $L$ is in \{1,2,3\}. $d_{ff}$ is the hidden size of feed-forward networks in each \method layer and in \{32,64,128,256,512\}. 
Also, batch size is 128, 128, 16, and 12 for ETT, Weather, Electricity, and Traffic datasets, respectively. Finally, we set the learning rate and training epochs to $10^{-3}$ and 100, respectively. Finally, we use Adam optimizer to train our model. The selected best hyperparameters of \method are in Table~\ref{tbl:hyper}. 

\begin{table}[t]
    \centering
    
    \caption{Selected hyperparameters of \method.}\label{tbl:hyper}
    \vspace{-7pt}
    \scriptsize
    \centering
    \renewcommand{\arraystretch}{0.9}
    
    \setlength{\tabcolsep}{17pt}
    \centering
    \resizebox{\columnwidth}{!}
    {
    \begin{tabular}{cc|ccccccc}
    \toprule
    Data & $\tau$ & $T$ & $N_S$ & $r_{\text{dropout}}$ & $d_h$ & $n_h$ & $L$ & $d_{ff}$ \\ \midrule     
        \multirow{4}{*}{\rotatebox{90}{ETTh1}} & 96 & 512 & 64 & 0.7 & 128 & 4 & 1 & 256 \\ 
        & 192 & 512 & 64 & 0.7 & 32 & 4 & 1 & 256 \\ 
        & 336 & 512 & 64 & 0.7 & 64 & 8 & 1 & 64 \\ 
        & 336 & 512 & 64 & 0.7 & 64 & 8 & 1 & 64 \\ 
        \midrule
        \multirow{4}{*}{\rotatebox{90}{ETTh2}} & 96 & 512 & 64 & 0.7 & 512 & 4 & 1 & 256 \\ 
        & 192 & 1024 & 64 & 0.7 & 512 & 2 & 1 & 256 \\ 
        & 336 & 1024 & 64 & 0.7 & 64 & 16 & 1 & 256 \\ 
        & 720 & 512 & 64 & 0.7 & 64 & 16 & 1 & 128 \\ 
        \midrule
        \multirow{4}{*}{\rotatebox{90}{ETTm1}} & 96 & 512 & 64 & 0.2 & 256 & 2 & 2 & 256 \\ 
        & 192 & 512 & 64 & 0.1 & 64 & 8 & 1 & 128 \\ 
        & 336 & 512 & 64 & 0.2 & 64 & 2 & 2 & 64 \\ 
        & 720 & 1024 & 64 & 0.7 & 64 & 4 & 1 & 128 \\ 
        \midrule
        \multirow{4}{*}{\rotatebox{90}{ETTm2}} & 96 & 1024 & 64 & 0.7 & 512 & 2 & 1 & 64 \\ 
        & 192 & 1024 & 64 & 0.7 & 128 & 4 & 1 & 32 \\ 
        & 336 & 1024 & 64 & 0.4 & 128 & 2 & 1 & 32 \\ 
        & 720 & 1024 & 64 & 0.7 & 256 & 4 & 1 & 32 \\ 
        \midrule
        \multirow{4}{*}{\rotatebox{90}{Weather}} & 96 & 512 & 64 & 0.2 & 128 & 8 & 3 & 256 \\ 
        & 192 & 512 & 64 & 0.2 & 128 & 16 & 3 & 256 \\
        & 336 & 512 & 64 & 0.4 & 128 & 16 & 3 & 512 \\
        & 720 & 512 & 64 & 0.4 & 128 & 2 & 1 & 256 \\ 
        \midrule
        \multirow{4}{*}{\rotatebox{90}{Electricity}} & 96 & 512 & 64 & 0.3 & 256 & 8 & 1 & 256 \\
        & 192 & 512 & 64 & 0.2 & 256 & 4 & 2 & 256 \\
        & 336 & 512 & 64 & 0.2 & 128 & 4 & 3 & 256 \\
        & 720 & 512 & 64 & 0.2 & 256 & 4 & 3 & 256 \\ 
        \midrule
        \multirow{4}{*}{\rotatebox{90}{Traffic}} & 96 & 512 & 8 & 0.2 & 512 & 2 & 3 & 512 \\
        & 192 & 512 & 8 & 0.1 & 256 & 4 & 3 & 512 \\
        & 336 & 512 & 8 & 0.2 & 256 & 2 & 3 & 256 \\
        & 720 & 512 & 8 & 0.2 & 512 & 4 & 3 & 512 \\ 
        \bottomrule
     
    \end{tabular}}
\end{table}

\section{Theoretical Complexity of Inter-Feature Attention in \method}\label{appen:complexity}
In this section, we elaborate on the reason why the theoretical complexity of inter-feature attention in \method is $\mathcal{O}(SD)$ where $D$ is the number of features and $S$ is the subset size. Attention cost in each subset is $\mathcal{O}(S^2)$. Because random partitioning generates $N_U \approx \frac{D}{S}$ subsets, the final complexity is $N_U \mathcal{O}(S^2) = \frac{D}{S}\mathcal{O}(S^2) = \mathcal{O}(SD)$.

\section{The Effect of Training \method with Random Sampling or Partitioning}\label{appen:random_sampling_partitioning}
In this section, we provide the experimental results where we train \method using a training algorithm with random sampling or partitioning (\textit{i.e.}, $use\_random\_partition=False$ or $True$ in Algorithm~\ref{alg:train2}). As shown in Table~\ref{tbl:random_sampling_partitioning}, these two ways are comparable in terms of forecasting performance --- note that we adopt the training algorithm based on random partitioning for our main experiments.

\begin{table}[!t]
    \caption{Test MSE and MAE of training \method using a training algorithm with random sampling or partitioning }\label{tbl:random_sampling_partitioning}
    \vspace{-3pt}
    \centering
    \scriptsize
    \renewcommand{\arraystretch}{1.0}
    \setlength{\tabcolsep}{3.7pt}
    \resizebox{\textwidth}{!}{%
    \begin{tabular}{c|c|cccc|cccc|cccc|cccc}
    \toprule
        \multirow{2}{*}{Score} & \multirow{2}{*}{Training Algorithm} & \multicolumn{4}{c|}{ETTh1 ($D=7$)} & \multicolumn{4}{c|}{ETTh2 (7)} & \multicolumn{4}{c|}{ETTm1 (7)} & \multicolumn{4}{c}{ETTm2 (7)} \\ 
        & & 96 & 192 & 336 & 720 & 96 & 192 & 336 & 720 & 96 & 192 & 336 & 720 & 96 & 192 & 336 & 720 \\ 
        \midrule
        \multirow{2}{*}{\rotatebox{0}{MSE}} & Random Partitioning & \textbf{0.361} & \textbf{0.396} & \textbf{0.400} & \textbf{0.412} & \textbf{0.269} & \textbf{0.323} & \textbf{0.317} & \textbf{0.370} & \textbf{0.282} & \textbf{0.325} & \textbf{0.352} & \textbf{0.401} & \textbf{0.160} & \textbf{0.213} & \textbf{0.262} & \textbf{0.336} \\ 
        & Random Sampling &  0.362 & 0.397 & \textbf{0.400} & \textbf{0.412} & 0.273 & \textbf{0.323} & \textbf{0.317} & 0.371 & 0.283 & \textbf{0.325} & \textbf{0.352} & 0.403 & 0.162 & 0.214 & 0.263 & 0.337\\ 
        \midrule
        \multirow{2}{*}{\rotatebox{0}{MAE}} & Random Partitioning & \textbf{0.390} & \textbf{0.414} & \textbf{0.421} & \textbf{0.442} & \textbf{0.332} & \textbf{0.369} & \textbf{0.378} & \textbf{0.416} & 0.340 & \textbf{0.365} & \textbf{0.385} & \textbf{0.408} & \textbf{0.253} & \textbf{0.290} & \textbf{0.325} & \textbf{0.372} \\ 
        & Random Sampling & 0.391 & 0.415 & \textbf{0.421} & \textbf{0.442} & 0.334 & \textbf{0.369} & 0.380 & \textbf{0.416} & \textbf{0.339} & \textbf{0.365} & \textbf{0.385} & 0.409 & 0.254 & 0.291 & 0.326 & 0.373 \\ 
        \bottomrule
    \end{tabular}}
    \setlength{\tabcolsep}{6.7pt}
    \resizebox{\textwidth}{!}{%
    \begin{tabular}{c|c|cccc|cccc|cccc}
    \toprule
        \multirow{2}{*}{Score} & \multirow{2}{*}{Training Algorithm} & \multicolumn{4}{c|}{Weather (21)} & \multicolumn{4}{c|}{Electricity (321)} & \multicolumn{4}{c}{Traffic (862)} \\ 
        & & 96 & 192 & 336 & 720 & 96 & 192 & 336 & 720 & 96 & 192 & 336 & 720 \\ 
        \midrule
        \multirow{2}{*}{\rotatebox{0}{MSE}} & Random Partitioning & \textbf{0.142} & 0.185 & \textbf{0.235} & \textbf{0.305} & \textbf{0.125} & 0.142 & \textbf{0.154} & \textbf{0.176} & \textbf{0.345} & \textbf{0.370} & \textbf{0.385} & \textbf{0.426} \\
        & Random Sampling & \textbf{0.142} & \textbf{0.184} & 0.237 & \textbf{0.305} & 0.126 & \textbf{0.141} & \textbf{0.154} & 0.180 & 0.347 & \textbf{0.370} & 0.386 & 0.427  \\
        \midrule
        \multirow{2}{*}{\rotatebox{0}{MAE}}  & Random Partitioning & \textbf{0.193} & 0.237 & \textbf{0.277} & \textbf{0.328} & \textbf{0.222} & 0.240 & 0.256 & \textbf{0.278} & \textbf{0.245} & \textbf{0.255} & \textbf{0.265} & \textbf{0.287} \\ 
        & Random Sampling & 0.195 & \textbf{0.236} & 0.278 & 0.329 & \textbf{0.222} & \textbf{0.239} & \textbf{0.255} & 0.281 & 0.246 & 0.256 & \textbf{0.265} & \textbf{0.287} \\
        \bottomrule
    \end{tabular}}

\end{table}

\section{The Performance of \method with $N_I=1$}\label{appen:ni1}
In Table~\ref{tbl:ni1}, we conduct the main experiments including \method with $N_I=1$ which is the number of repeating an inference process based on random partitioning. In this experiment, we include some baselines showing decent forecasting performance. As Table~\ref{tbl:ni1} shows, despite $N_I=1$, \method still gives better results than baselines.

\begin{table}[t]
    \centering
    \caption{MSE scores of main forecasting results including \method wiht $N_I=1$.}\label{tbl:ni1}
    \vspace{-3pt}
    \scriptsize
    \renewcommand{\arraystretch}{0.9}
    \setlength{\tabcolsep}{4.5pt}
    \centering
    \resizebox{\columnwidth}{!}
    {%
    \begin{tabular}{cc|c|cccc|c|cccc}
    \toprule
    \multicolumn{2}{c|}{\multirow{2}{*}{Data}} & \multicolumn{5}{c|}{MSE} & \multicolumn{5}{c}{MAE} \\ 
    & & \method & PatchTST & Nlinear & TSMIxer & DeepTime & \method & PatchTST & Nlinear & TSMIxer & DeepTime \\ \midrule
    \multirow{4}{*}{\rotatebox{90}{ETTh1}} & 96 & \textbf{0.361} & \underline{0.370} & 0.374 & \textbf{0.361} & 0.372 & \textbf{0.391} & 0.400 & 0.394 & \underline{0.392} & 0.398 \\ 
& 192 & \textbf{0.393} & 0.413 & 0.408 & \underline{0.404} & 0.405 & \textbf{0.409} & 0.429 & \underline{0.415} & 0.418 & 0.419 \\
& 336 & \textbf{0.404} & 0.422 & 0.429 & \underline{0.420} & 0.437 & \textbf{0.417} & 0.440 & \underline{0.427} & 0.431 & 0.442 \\
& 720 & \textbf{0.412} & 0.447 & \underline{0.440} & 0.463 & 0.477 & \textbf{0.442} & 0.468 & \underline{0.453} & 0.472 & 0.493 \\ \midrule
\multirow{4}{*}{\rotatebox{90}{ETTh2}} & 96 & \textbf{0.270} & \underline{0.274} & 0.277 & \underline{0.274} & 0.291 & \textbf{0.332} & \underline{0.337} & 0.338 & 0.341 & 0.350 \\
& 192 & \textbf{0.321} & 0.341 & 0.344 & \underline{0.339} & 0.403 & \textbf{0.369} & 0.382 & \underline{0.381} & 0.385 & 0.427 \\
& 336 & \textbf{0.317} & \underline{0.329} & 0.357 & 0.361 & 0.466 & \textbf{0.380} & \underline{0.384} & 0.400 & 0.406 & 0.475 \\
& 720 & \textbf{0.371} & \underline{0.379} & 0.394 & 0.445 & 0.576 & \textbf{0.416} & \underline{0.422} & 0.436 & 0.470 & 0.545 \\ \midrule
\multirow{4}{*}{\rotatebox{90}{ETTm1}} & 96 & \underline{0.286} & 0.293 & 0.306 & \textbf{0.285} & 0.311 & \underline{0.343} & 0.346 & 0.348 & \textbf{0.339} & 0.353 \\
& 192 & \underline{0.328} & 0.333 & 0.349 & \textbf{0.327} & 0.339 & \underline{0.368} & 0.370 & 0.375 & \textbf{0.365} & 0.369 \\
& 336 & \textbf{0.354} & 0.369 & 0.375 & \underline{0.356} & 0.366 & \underline{0.387} & 0.392 & 0.388 & \textbf{0.382} & 0.391 \\
& 720 & \underline{0.403} & 0.416 & 0.433 & 0.419 & \textbf{0.400} & \textbf{0.409} & 0.420 & 0.422 & \underline{0.414} & \underline{0.414} \\ \midrule
\multirow{4}{*}{\rotatebox{90}{ETTm2}} & 96 & \textbf{0.160} & 0.166 & 0.167 & \underline{0.163} & 0.165 & \textbf{0.250} & 0.256 & 0.255 & \underline{0.252} & 0.259 \\
& 192 & \textbf{0.213} & 0.223 & 0.221 & \underline{0.216} & 0.222 & \textbf{0.288} & 0.296 & 0.293 & \underline{0.290} & 0.299 \\
& 336 & \textbf{0.262} & 0.274 & 0.274 & \underline{0.268} & 0.278 & \underline{0.325} & 0.329 & 0.327 & \textbf{0.324} & 0.338 \\
& 720 & \textbf{0.336} & \underline{0.361} & 0.368 & 0.420 & 0.369 & \textbf{0.372} & 0.394 & \underline{0.384} & 0.422 & 0.400 \\ \midrule
\multirow{4}{*}{\rotatebox{90}{Weather}} & 96 & \textbf{0.142} & 0.149 & 0.182 & \underline{0.145} & 0.169 & \textbf{0.194} & \underline{0.198} & 0.232 & \underline{0.198} & 0.227 \\
& 192 & \textbf{0.186} & 0.194 & 0.225 & \underline{0.191} & 0.211 & \textbf{0.238} & \underline{0.241} & 0.269 & 0.242 & 0.266 \\
& 336 & \textbf{0.236} & 0.245 & 0.271 & \underline{0.242} & 0.255 & \textbf{0.278} & 0.282 & 0.301 & \underline{0.280} & 0.304 \\
& 720 & \textbf{0.305} & \underline{0.314} & 0.338 & 0.320 & 0.318 & \textbf{0.328} & \underline{0.334} & 0.348 & 0.336 & 0.357 \\ \midrule
\multirow{4}{*}{\rotatebox{90}{Electricity}} & 96 & \textbf{0.127} & \underline{0.129} & 0.141 & 0.131 & 0.139 & \underline{0.224} & \textbf{0.222} & 0.237 & 0.229 & 0.239 \\
& 192 & \textbf{0.145} & \underline{0.147} & 0.154 & 0.151 & 0.154 & \underline{0.244} & \textbf{0.240} & 0.248 & 0.246 & 0.253 \\
& 336 & \textbf{0.158} & 0.163 & 0.171 & \underline{0.161} & 0.169 & \underline{0.260} & \textbf{0.259} & 0.265 & 0.261 & 0.270 \\
& 720 & \textbf{0.181} & \underline{0.197} & 0.210 & \underline{0.197} & 0.201 & \textbf{0.283} & \underline{0.290} & 0.297 & 0.293 & 0.300 \\ \midrule
\multirow{4}{*}{\rotatebox{90}{Traffic}} & 96 & \textbf{0.347} & \underline{0.360} & 0.410 & 0.376 & 0.401 & \textbf{0.247} & \underline{0.249} & 0.279 & 0.264 & 0.280 \\
& 192 & \textbf{0.372} & \underline{0.379} & 0.423 & 0.397 & 0.413 & \underline{0.258} & \textbf{0.256} & 0.284 & 0.277 & 0.285 \\
& 336 & \textbf{0.387} & \underline{0.392} & 0.435 & 0.413 & 0.425 & \underline{0.267} & \textbf{0.264} & 0.290 & 0.290 & 0.292 \\
& 720 & \textbf{0.430} & \underline{0.432} & 0.464 & 0.444 & 0.462 & \underline{0.289} & \textbf{0.286} & 0.307 & 0.306 & 0.312 \\ \midrule
\multicolumn{2}{c|}{Avg. Rank} & \textbf{1.107} & 2.786 & 4.321 & \underline{2.571} & 4.036 & \textbf{1.357} & \underline{2.714} & 3.464 & 2.750 & 4.607 \\
    \bottomrule
    \end{tabular}}
\end{table}

\section{Additional Experiments}\label{appen:add}

\subsection{Additional Experimental Results in Tabular Forms}\label{appen:add_result}
In this section, we provide additional results for existing experiments, such as experiments with other datasets and MAE evaluation metrics. Table~\ref{tbl:main_appen}, Table~\ref{tbl:ablation_appen}, and Table~\ref{tbl:various_inference_appen} are additional results for Table~\ref{tbl:main}, Table~\ref{tbl:ablation}, and Table~\ref{tbl:various_inference}, respectively. Furthermore, Table~\ref{tbl:main_std} provides the standard deviation information of \method in forecasting accuracy of Table~\ref{tbl:main_appen}.

\begin{table}[!t]
    \centering
    
    \caption{MSE and MAE scores of main forecasting results. `former' included in some model names is abbreviated to `f.'. Also,  `P.M.', , `U.', and `C.M.' denote partial-multivariate, univariate, and complete-multivariate, respectively. (Additional results for Table~\ref{tbl:main})}\label{tbl:main_appen}
    \vspace{-3pt}
    \scriptsize
    \renewcommand{\arraystretch}{0.9}
    \setlength{\tabcolsep}{2pt}
    \centering
    \resizebox{\textwidth}{!}
    {%
    \begin{tabular}{c|cc|c|ccc|cccccccc}
    \toprule
    
        \multirow{2}{*}{Score} & \multicolumn{2}{c|}{\multirow{2}{*}{Data}}  & P.M. & \multicolumn{3}{c|}{U.} & \multicolumn{8}{c}{C.M.} \\ 
        ~ & ~ & ~ & PMf. & PatchTST & Dlinear & Nlinear & Crossf. & FEDf. & Inf. & Pyraf. & TSMixer & DeepTime & MICN & TimesNet \\
        \midrule
        \multirow{29}[8]{*}{\rotatebox{0}{MSE}} &\multirow{4}{*}{\rotatebox{90}{ETTh1}} & 96 & \textbf{0.361} & \underline{0.370} & 0.375 & 0.374 & 0.427 & 0.376 & 0.941 & 0.664 & \textbf{0.361} & 0.372 & 0.828 & 0.465\\
 & & 192 & \textbf{0.396} & 0.413 & 0.405 & 0.408 & 0.537 & 0.423 & 1.007 & 0.790 & \underline{0.404} & 0.405 & 0.765 & 0.493\\
 & & 336 & \textbf{0.400} & 0.422 & 0.439 & 0.429 & 0.651 & 0.444 & 1.038 & 0.891 & \underline{0.420} & 0.437 & 0.904 & 0.456\\
 & & 720 & \textbf{0.412} & 0.447 & 0.472 & \underline{0.440} & 0.664 & 0.469 & 1.144 & 0.963 & 0.463 & 0.477 & 1.192 & 0.533\\ \cmidrule{2-15}
 & \multirow{4}{*}{\rotatebox{90}{ETTh2}} & 96 & \textbf{0.269} & \underline{0.274} & 0.289 & 0.277 & 0.720 & 0.332 & 1.549 & 0.645 & \underline{0.274} & 0.291 & 0.452 & 0.381\\
 & & 192 & \textbf{0.323} & 0.341 & 0.383 & 0.344 & 1.121 & 0.407 & 3.792 & 0.788 & \underline{0.339} & 0.403 & 0.554 & 0.416\\
 & & 336 & \textbf{0.317} & \underline{0.329} & 0.448 & 0.357 & 1.524 & 0.400 & 4.215 & 0.907 & 0.361 & 0.466 & 0.582 & 0.363\\
 & & 720 & \textbf{0.370} & 0.379 & 0.605 & 0.394 & 3.106 & 0.412 & 3.656 & 0.963 & 0.445 & 0.576 & 0.869 & \underline{0.371}\\ \cmidrule{2-15}
 & \multirow{4}{*}{\rotatebox{90}{ETTm1}} & 96 & \textbf{0.282} & 0.293 & 0.299 & 0.306 & 0.336 & 0.326 & 0.626 & 0.543 & \underline{0.285} & 0.311 & 0.406 & 0.343\\
 & & 192 & \textbf{0.325} & 0.333 & 0.335 & 0.349 & 0.387 & 0.365 & 0.725 & 0.557 & \underline{0.327} & 0.339 & 0.500 & 0.381\\
 & & 336 & \textbf{0.352} & 0.369 & 0.369 & 0.375 & 0.431 & 0.392 & 1.005 & 0.754 & \underline{0.356} & 0.366 & 0.580 & 0.436\\
 & & 720 & \underline{0.401} & 0.416 & 0.425 & 0.433 & 0.555 & 0.446 & 1.133 & 0.908 & 0.419 & \textbf{0.400} & 0.607 & 0.527\\ \cmidrule{2-15}
 & \multirow{4}{*}{\rotatebox{90}{ETTm2}} & 96 & \textbf{0.160} & 0.166 & 0.167 & 0.167 & 0.338 & 0.180 & 0.355 & 0.435 & \underline{0.163} & 0.165 & 0.238 & 0.218\\
 & & 192 & \textbf{0.213} & 0.223 & 0.224 & 0.221 & 0.567 & 0.252 & 0.595 & 0.730 & \underline{0.216} & 0.222 & 0.302 & 0.282\\
 & & 336 & \textbf{0.262} & 0.274 & 0.281 & 0.274 & 1.050 & 0.324 & 1.270 & 1.201 & \underline{0.268} & 0.278 & 0.447 & 0.378\\
 & & 720 & \textbf{0.336} & \underline{0.361} & 0.397 & 0.368 & 2.049 & 0.410 & 3.001 & 3.625 & 0.420 & 0.369 & 0.549 & 0.444\\ \cmidrule{2-15}
 & \multirow{4}{*}{\rotatebox{90}{Weather}} & 96 & \textbf{0.142} & 0.149 & 0.176 & 0.182 & 0.150 & 0.238 & 0.354 & 0.896 & \underline{0.145} & 0.169 & 0.188 & 0.179\\
 & & 192 & \textbf{0.185} & 0.194 & 0.220 & 0.225 & 0.200 & 0.275 & 0.419 & 0.622 & \underline{0.191} & 0.211 & 0.231 & 0.230\\
 & & 336 & \textbf{0.235} & 0.245 & 0.265 & 0.271 & 0.263 & 0.339 & 0.583 & 0.739 & \underline{0.242} & 0.255 & 0.280 & 0.276\\
 & & 720 & \textbf{0.305} & 0.314 & 0.323 & 0.338 & \underline{0.310} & 0.389 & 0.916 & 1.004 & 0.320 & 0.318 & 0.358 & 0.347\\ \cmidrule{2-15}
 & \multirow{4}{*}{\rotatebox{90}{Electricity}} & 96 & \textbf{0.125} & \underline{0.129} & 0.140 & 0.141 & 0.135 & 0.186 & 0.304 & 0.386 & 0.131 & 0.139 & 0.177 & 0.186\\
 & & 192 & \textbf{0.142} & \underline{0.147} & 0.153 & 0.154 & 0.158 & 0.197 & 0.327 & 0.386 & 0.151 & 0.154 & 0.195 & 0.208\\
 & & 336 & \textbf{0.154} & 0.163 & 0.169 & 0.171 & 0.177 & 0.213 & 0.333 & 0.378 & \underline{0.161} & 0.169 & 0.213 & 0.210\\
 & & 720 & \textbf{0.176} & \underline{0.197} & 0.203 & 0.210 & 0.222 & 0.233 & 0.351 & 0.376 & \underline{0.197} & 0.201 & 0.204 & 0.231\\ \cmidrule{2-15}
 & \multirow{4}{*}{\rotatebox{90}{Traffic}} & 96 & \textbf{0.345} & \underline{0.360} & 0.410 & 0.410 & 0.481 & 0.576 & 0.733 & 2.085 & 0.376 & 0.401 & 0.489 & 0.599\\
 & & 192 & \textbf{0.370} & \underline{0.379} & 0.423 & 0.423 & 0.509 & 0.610 & 0.777 & 0.867 & 0.397 & 0.413 & 0.493 & 0.612\\
 & & 336 & \textbf{0.385} & \underline{0.392} & 0.436 & 0.435 & 0.534 & 0.608 & 0.776 & 0.869 & 0.413 & 0.425 & 0.496 & 0.618\\
 & & 720 & \textbf{0.426} & \underline{0.432} & 0.466 & 0.464 & 0.585 & 0.621 & 0.827 & 0.881 & 0.444 & 0.462 & 0.520 & 0.654\\ 
 \cmidrule{2-15}
        & \multicolumn{2}{c|}{Avg. Rank} & \textbf{1.036} & 2.893 & 5.357 & 5.071 & 8.036 & 7.821 & 11.429 & 11.286 & \underline{2.821} & 4.607 & 9.000 & 8.179 \\ 
        \midrule
        \multirow{29}[8]{*}{\rotatebox{0}{MAE}} & \multirow{4}{*}{\rotatebox{90}{ETTh1}} & 96 & \textbf{0.390} & 0.400 & 0.399 & 0.394 & 0.448 & 0.415 & 0.769 & 0.612 & \underline{0.392} & 0.398 & 0.607 & 0.466\\
 & & 192 & \textbf{0.414} & 0.429 & 0.416 & \underline{0.415} & 0.520 & 0.446 & 0.786 & 0.681 & 0.418 & 0.419 & 0.575 & 0.479\\
 & & 336 & \textbf{0.421} & 0.440 & 0.443 & \underline{0.427} & 0.588 & 0.462 & 0.784 & 0.738 & 0.431 & 0.442 & 0.621 & 0.473\\
 & & 720 & \textbf{0.442} & 0.468 & 0.490 & \underline{0.453} & 0.612 & 0.492 & 0.857 & 0.782 & 0.472 & 0.493 & 0.736 & 0.525\\ \cmidrule{2-15}
 & \multirow{4}{*}{\rotatebox{90}{ETTh2}} & 96 & \textbf{0.332} & \underline{0.337} & 0.353 & 0.338 & 0.615 & 0.374 & 0.952 & 0.597 & 0.341 & 0.350 & 0.462 & 0.423\\
 & & 192 & \textbf{0.369} & 0.382 & 0.418 & \underline{0.381} & 0.785 & 0.446 & 1.542 & 0.683 & 0.385 & 0.427 & 0.528 & 0.445\\
 & & 336 & \textbf{0.378} & \underline{0.384} & 0.465 & 0.400 & 0.980 & 0.447 & 1.642 & 0.747 & 0.406 & 0.475 & 0.556 & 0.422\\
 & & 720 & \textbf{0.416} & \underline{0.422} & 0.551 & 0.436 & 1.487 & 0.469 & 1.619 & 0.783 & 0.470 & 0.545 & 0.667 & 0.424\\ \cmidrule{2-15}
 & \multirow{4}{*}{\rotatebox{90}{ETTm1}} & 96 & \underline{0.340} & 0.346 & 0.343 & 0.348 & 0.387 & 0.390 & 0.560 & 0.510 & \textbf{0.339} & 0.353 & 0.434 & 0.381\\
 & & 192 & \textbf{0.365} & 0.370 & \textbf{0.365} & 0.375 & 0.419 & 0.415 & 0.619 & 0.537 & \textbf{0.365} & \underline{0.369} & 0.500 & 0.403\\
 & & 336 & \underline{0.385} & 0.392 & 0.386 & 0.388 & 0.449 & 0.425 & 0.741 & 0.655 & \textbf{0.382} & 0.391 & 0.549 & 0.438\\
 & & 720 & \textbf{0.408} & 0.420 & 0.421 & 0.422 & 0.532 & 0.458 & 0.845 & 0.724 & \underline{0.414} & \underline{0.414} & 0.560 & 0.488\\ \cmidrule{2-15}
 & \multirow{4}{*}{\rotatebox{90}{ETTm2}} & 96 & \underline{0.253} & 0.256 & 0.260 & 0.255 & 0.393 & 0.271 & 0.462 & 0.507 & \textbf{0.252} & 0.259 & 0.331 & 0.307\\
 & & 192 & \textbf{0.290} & 0.296 & 0.303 & \underline{0.293} & 0.519 & 0.318 & 0.586 & 0.673 & \textbf{0.290} & 0.299 & 0.374 & 0.352\\
 & & 336 & \underline{0.325} & 0.329 & 0.342 & 0.327 & 0.732 & 0.364 & 0.871 & 0.845 & \textbf{0.324} & 0.338 & 0.478 & 0.407\\
 & & 720 & \textbf{0.372} & 0.394 & 0.421 & \underline{0.384} & 1.170 & 0.420 & 1.267 & 1.451 & 0.422 & 0.400 & 0.554 & 0.450\\ \cmidrule{2-15}
 & \multirow{4}{*}{\rotatebox{90}{Weather}} & 96 & \textbf{0.193} & \underline{0.198} & 0.237 & 0.232 & 0.224 & 0.314 & 0.405 & 0.556 & \underline{0.198} & 0.227 & 0.258 & 0.237\\
 & & 192 & \textbf{0.237} & \underline{0.241} & 0.282 & 0.269 & 0.267 & 0.329 & 0.434 & 0.624 & 0.242 & 0.266 & 0.295 & 0.279\\
 & & 336 & \textbf{0.277} & 0.282 & 0.319 & 0.301 & 0.328 & 0.377 & 0.543 & 0.753 & \underline{0.280} & 0.304 & 0.337 & 0.310\\
 & & 720 & \textbf{0.328} & \underline{0.334} & 0.362 & 0.348 & 0.363 & 0.409 & 0.705 & 0.934 & 0.336 & 0.357 & 0.399 & 0.353\\ \cmidrule{2-15}
 & \multirow{4}{*}{\rotatebox{90}{Electricity}} & 96 & \textbf{0.222} & \textbf{0.222} & 0.237 & 0.237 & 0.234 & 0.302 & 0.393 & 0.449 & \underline{0.229} & 0.239 & 0.294 & 0.290\\
 & & 192 & \textbf{0.240} & \textbf{0.240} & 0.249 & 0.248 & 0.262 & 0.311 & 0.417 & 0.443 & \underline{0.246} & 0.253 & 0.306 & 0.301\\
 & & 336 & \textbf{0.256} & \underline{0.259} & 0.267 & 0.265 & 0.283 & 0.328 & 0.422 & 0.443 & 0.261 & 0.270 & 0.324 & 0.314\\
 & & 720 & \textbf{0.278} & \underline{0.290} & 0.301 & 0.297 & 0.328 & 0.344 & 0.427 & 0.445 & 0.293 & 0.300 & 0.317 & 0.329\\ \cmidrule{2-15}
 & \multirow{4}{*}{\rotatebox{90}{Traffic}} & 96 & \textbf{0.245} & \underline{0.249} & 0.282 & 0.279 & 0.265 & 0.359 & 0.410 & 0.468 & 0.264 & 0.280 & 0.317 & 0.325\\
 & & 192 & \textbf{0.255} & \underline{0.256} & 0.287 & 0.284 & 0.277 & 0.380 & 0.435 & 0.467 & 0.277 & 0.285 & 0.319 & 0.332\\
 & & 336 & \underline{0.265} & \textbf{0.264} & 0.296 & 0.290 & 0.291 & 0.375 & 0.434 & 0.469 & 0.290 & 0.292 & 0.317 & 0.332\\
 & & 720 & \underline{0.287} & \textbf{0.286} & 0.315 & 0.307 & 0.325 & 0.375 & 0.466 & 0.473 & 0.306 & 0.312 & 0.326 & 0.348\\ \cmidrule{2-15}
    & \multicolumn{2}{c|}{Avg. Rank} & \textbf{1.214} & 2.964 & 5.643 & 3.821 & 8.000 & 8.214 & 11.464 & 11.393 & \underline{2.857} & 5.357 & 9.071 & 7.571 \\  
 
        \bottomrule
    \end{tabular}}
\end{table}
\begin{table}[!t]
    \caption{Test MSE and MAE of three types of models by adjusting $S$ in \method. (Additional results for Table~\ref{tbl:ablation})}\label{tbl:ablation_appen}
    \vspace{-3pt}
    \centering
    \scriptsize
    \renewcommand{\arraystretch}{0.9}
    \setlength{\tabcolsep}{1.3pt}
    \resizebox{\textwidth}{!}{%
    \begin{tabular}{c|cc|cccc|cccc|cccc|cccc}
    \toprule
        \multirow{2}{*}{Score} & \multicolumn{2}{c|}{\multirow{2}{*}{\method Variants}} & \multicolumn{4}{c|}{ETTh1 ($D=7$)} & \multicolumn{4}{c|}{ETTh2 (7)} & \multicolumn{4}{c|}{ETTm1 (7)} & \multicolumn{4}{c}{ETTm2 (7)} \\ 
        & & & 96 & 192 & 336 & 720 & 96 & 192 & 336 & 720 & 96 & 192 & 336 & 720 & 96 & 192 & 336 & 720 \\ 
        \midrule
        \multirow{3}{*}{\rotatebox{90}{MSE}} & Univariate & $S=1$ & \textbf{0.361} & \textbf{0.393} & 0.404 & 0.420 & \underline{0.272} & \underline{0.325} & \underline{0.318} & \underline{0.371} & \underline{0.288} & \underline{0.335} & \underline{0.358} & 0.403 & \underline{0.161} & \textbf{0.213} & \underline{0.265} & \underline{0.338} \\
        & Partial-Multivariate & $1<S<D$ & \textbf{0.361} & 0.396 & \textbf{0.400} & \textbf{0.412} & \textbf{0.269} & \textbf{0.323} & \textbf{0.317} & \textbf{0.370} & \textbf{0.282} & \textbf{0.325} & \textbf{0.352} & \textbf{0.401} & \textbf{0.160} & \textbf{0.213} & \textbf{0.262} & \textbf{0.336} \\
        & Complete-Multivariate & $S=D$ & \textbf{0.361} & \underline{0.395} & \underline{0.401} & \underline{0.413} & \textbf{0.269} & \underline{0.325} & \underline{0.318} & \underline{0.371} & 0.299 & 0.350 & 0.377 & \underline{0.402} & \underline{0.161} & \textbf{0.213} & \underline{0.265} & \underline{0.338} \\
        \midrule
        \multirow{3}{*}{\rotatebox{90}{MAE}} & Univariate & $S=1$ & \textbf{0.390} & \textbf{0.410} & \textbf{0.419} & \underline{0.446} & \underline{0.334} & 0.373 & \underline{0.380} & \underline{0.418} & \underline{0.344} & \underline{0.371} & \underline{0.386} & \underline{0.409} & \textbf{0.253} & \textbf{0.290} & 0.328 & 0.376 \\
         & Partial-Multivariate & $1<S<D$ & \textbf{0.390} & 0.414 & 0.421 & \textbf{0.442} & \textbf{0.332} & \textbf{0.369} & \textbf{0.378} & \textbf{0.416} & \textbf{0.340} & \textbf{0.365} & \textbf{0.385} & \textbf{0.408} & \textbf{0.253} & \textbf{0.290} & \textbf{0.325} & \textbf{0.372} \\
        & Complete-Multivariate & $S=D$ & \textbf{0.390} & \underline{0.413} & \underline{0.420} & \textbf{0.442} & \textbf{0.332} & \underline{0.371} & \underline{0.380} & \textbf{0.416} & 0.353 & 0.382 & 0.396 & \textbf{0.408} & \textbf{0.253} & \textbf{0.290} & \underline{0.327} & \underline{0.374} \\
        \bottomrule
    \end{tabular}}
    \renewcommand{\arraystretch}{0.9}
    \setlength{\tabcolsep}{3.5pt}
    \resizebox{\textwidth}{!}{%
    \begin{tabular}{c|cc|cccc|cccc|cccc}
    \toprule
        \multirow{2}{*}{Score} & \multicolumn{2}{c|}{\multirow{2}{*}{\method Variants}} & \multicolumn{4}{c|}{Weather (21)} & \multicolumn{4}{c|}{Electricity (321)} & \multicolumn{4}{c}{Traffic (862)} \\ 
        & & & 96 & 192 & 336 & 720 & 96 & 192 & 336 & 720 & 96 & 192 & 336 & 720 \\ 
        \midrule
        \multirow{3}{*}{\rotatebox{90}{MSE}} & Univariate & $S=1$ & \textbf{0.141} & \underline{0.186} & \underline{0.237} & 0.308 & \underline{0.128} & \underline{0.146} & \underline{0.163} & \underline{0.204} & 0.368 & 0.388 & 0.404 & \underline{0.441} \\ 
        & Partial-Multivariate & $1<S<D$ & \underline{0.142} & \textbf{0.185} & \textbf{0.235} & \textbf{0.305} & \textbf{0.125} & \textbf{0.142} & \textbf{0.154} & \textbf{0.176} & \textbf{0.345} & \textbf{0.370} & \textbf{0.385} & \textbf{0.426} \\
        & Complete-Multivariate & $S=D$ & 0.146 & 0.192 & 0.244  & \underline{0.307} & 0.129 & 0.147 & \underline{0.163} & \underline{0.204} & \underline{0.363} & \underline{0.383} & \underline{0.394} & \underline{0.441} \\
        \midrule
        \multirow{3}{*}{\rotatebox{90}{MAE}} & Univariate & $S=1$ & \underline{0.195} & \underline{0.239} & \underline{0.279} & 0.333 & \underline{0.223} & \underline{0.242} & 0.260 & \underline{0.297} & \underline{0.257} & \underline{0.265} & 0.277 & \underline{0.299} \\
        & Partial-Multivariate & $1<S<D$ & \textbf{0.193} & \textbf{0.237} & \textbf{0.277} & \textbf{0.328} & \textbf{0.222} & \textbf{0.240} & \textbf{0.256} & \textbf{0.278} & \textbf{0.245} & \textbf{0.255} & \textbf{0.265} & \textbf{0.287} \\ 
        & Complete-Multivariate & $S=D$ & 0.199 & 0.243 & 0.286 & \underline{0.331} & 0.228 & 0.246 & \underline{0.259} & \underline{0.297} & \underline{0.257} & 0.269 & \underline{0.276} & 0.303 \\
        \bottomrule
    \end{tabular}}

\end{table}

\begin{table}[!t]
    \centering
    \caption{Test MSE and MAE of \method with various inference techniques. (Additional results for Table~\ref{tbl:various_inference}) }\label{tbl:various_inference_appen}
    \vspace{-3pt}
    \scriptsize
    \renewcommand{\arraystretch}{0.99}
    \setlength{\tabcolsep}{5pt}
    \resizebox{\textwidth}{!}{
    \begin{tabular}{c|c|cccc|cccc}
    \toprule
        \multirow{2}{*}{Score} & \multirow{2}{*}{\makecell{Inference Technique}} & \multicolumn{4}{c|}{Electricity ($D=$ 321)} & \multicolumn{4}{c}{Traffic (862)} \\ 
        & & $\tau$=96 & 192 & 336 & 720 & 96 & 192 & 336 & 720 \\ 
        \midrule
        \multirow{3}{*}{MSE} & Proposed Technique with $N_I=3$ (Ours) & \textbf{0.125} & \textbf{0.142} & \textbf{0.154} & \textbf{0.176} & \textbf{0.345} & \textbf{0.370} & \textbf{0.385} & \textbf{0.426} \\ 
        & Sampling A Subset of Mutually Significant Features & \underline{0.132} & \underline{0.148} & 0.178 & \underline{0.205} & \underline{0.352} & \underline{0.372} & \underline{0.386} & \underline{0.428} \\
        & Sampling A Subset of Mutually Insignificant Features & 0.135 & 0.167 & \underline{0.174} & 0.235 & 0.377 & 0.410 & 0.410 & 0.444\\
        \midrule
        \multirow{3}{*}{MAE} & Proposed Technique with $N_I=3$ (Ours) & \textbf{0.222} & \textbf{0.240} & \textbf{0.256} & \textbf{0.278} & \textbf{0.245} & \textbf{0.255} & \textbf{0.265} & \textbf{0.287} \\
        & Sampling A Subset of Mutually Significant Features & \underline{0.231} & \underline{0.247} & 0.285 & \underline{0.302} & \underline{0.251} & \underline{0.259} & \underline{0.267} & \underline{0.289} \\
        & Sampling A Subset of Mutually Insignificant Features & 0.237 & 0.268 & \underline{0.276} & 0.329 &0.267 & 0.285 & 0.289 & 0.308 \\
        \bottomrule
    \end{tabular}}
\end{table}
\begin{table}[!t]
    \caption{Main forecasting results of \method with standard deviation}\label{tbl:main_std}
    \vspace{-3pt}
    \centering
    \footnotesize
    \renewcommand{\arraystretch}{0.9}
    \setlength{\tabcolsep}{4pt}
    \resizebox{\textwidth}{!}{%
    \begin{tabular}{c|c|cccc|cccc|cccc|cccc}
    \toprule
        \multirow{8}{*}{\rotatebox{90}{\method}} & \multirow{2}{*}{Score} & \multicolumn{4}{c|}{ETTh1 ($D=7$)} & \multicolumn{4}{c|}{ETTh2 (7)} & \multicolumn{4}{c|}{ETTm1 (7)} & \multicolumn{4}{c}{ETTm2 (7)} \\ 
        & & 96 & 192 & 336 & 720 & 96 & 192 & 336 & 720 & 96 & 192 & 336 & 720 & 96 & 192 & 336 & 720 \\ 
        \cmidrule{2-18}
        & \multirow{2}{*}{\rotatebox{0}{MSE}} & 0.361 & 0.396 & 0.400 & 0.412 & 0.269 & 0.323 & 0.317 & 0.370 & 0.282 & 0.325 & 0.352 & 0.401 & 0.160 & 0.213 & 0.262 & 0.336 \\
        & & { \tiny $\pm$0.001} & { \tiny $\pm$0.002} & { \tiny $\pm$0.001} & { \tiny $\pm$0.001} & { \tiny $\pm$0.001} & { \tiny $\pm$0.002} & { \tiny $\pm$0.001} & { \tiny $\pm$0.001} & { \tiny $\pm$0.002} & { \tiny $\pm$0.001} & { \tiny $\pm$0.001} & { \tiny $\pm$0.000} & { \tiny $\pm$0.001} & { \tiny $\pm$0.001} & { \tiny $\pm$0.002} & { \tiny $\pm$0.001}  \\ 
        \cmidrule{2-18}
        & \multirow{2}{*}{\rotatebox{0}{MAE}} & 0.390 & 0.414 & 0.421 & 0.442 & 0.332 & 0.369 & 0.378 & 0.416 & 0.340 & 0.365 & 0.385 & 0.408 & 0.253 & 0.290 & 0.325 & 0.372 \\
        & & { \tiny $\pm$0.001} & { \tiny $\pm$0.002} & { \tiny $\pm$0.001} & { \tiny $\pm$0.002} & { \tiny $\pm$0.001} & { \tiny $\pm$0.002} & { \tiny $\pm$0.001} & { \tiny $\pm$0.001} & { \tiny $\pm$0.001} & { \tiny $\pm$0.001} & { \tiny $\pm$0.000} & { \tiny $\pm$0.000} & { \tiny $\pm$0.001} & { \tiny $\pm$0.001} & { \tiny $\pm$0.001} & { \tiny $\pm$0.001} \\
        \bottomrule
    \end{tabular}}
    \renewcommand{\arraystretch}{0.9}
    \setlength{\tabcolsep}{8pt}
    \resizebox{\textwidth}{!}{%
    \begin{tabular}{c|c|cccc|cccc|cccc}
    \toprule
        \multirow{8}{*}{\rotatebox{90}{\method}} & \multirow{2}{*}{Score} & \multicolumn{4}{c|}{Weather (21)} & \multicolumn{4}{c|}{Electricity (321)} & \multicolumn{4}{c}{Traffic (862)} \\ 
        & & 96 & 192 & 336 & 720 & 96 & 192 & 336 & 720 & 96 & 192 & 336 & 720 \\ 
        \cmidrule{2-14}
        & \multirow{2}{*}{\rotatebox{0}{MSE}} &  0.142 & 0.185 & 0.235 & 0.305 & 0.125 & 0.142 & 0.154 & 0.176 & 0.345 & 0.370 & 0.385 & 0.426 \\
        & & { \tiny $\pm$0.000} & { \tiny $\pm$0.000} & { \tiny $\pm$0.001} & { \tiny $\pm$0.001} & { \tiny $\pm$0.000} & { \tiny $\pm$0.001} & { \tiny $\pm$0.001} & { \tiny $\pm$0.003} & { \tiny $\pm$0.001} & { \tiny $\pm$0.001} & { \tiny $\pm$0.001} & { \tiny $\pm$0.001}\\
        \cmidrule{2-14}
        & \multirow{2}{*}{\rotatebox{0}{MSE}} &  0.193 & 0.237 & 0.277 & 0.328 & 0.222 & 0.240 & 0.256 & 0.278 & 0.245 & 0.255 & 0.265 & 0.287 \\
        & & { \tiny $\pm$0.001} & { \tiny $\pm$0.000} & { \tiny $\pm$0.001} & { \tiny $\pm$0.001} & { \tiny $\pm$0.001} & { \tiny $\pm$0.001} & { \tiny $\pm$0.000} & { \tiny $\pm$0.003} & { \tiny $\pm$0.001} & { \tiny $\pm$0.000} & { \tiny $\pm$0.000} & { \tiny $\pm$0.001}
        \\
        \bottomrule
    \end{tabular}}

\end{table}



\subsection{Additional Visualization}\label{appen:add_visual}
Like Appendix~\ref{appen:add_result}, this section provides additional visualizations with other datasets or models for existing ones. 
Figure~\ref{fig:change_s_appen} is for Figure~\ref{fig:change_s},
Figure~\ref{fig:change_f_all_appen} for Figure~\ref{fig:change_f_all},
Figure~\ref{fig:sens_to_ni_a_appen} for Figure~\ref{fig:sens_to_ni}(a),
Figure~\ref{fig:sens_to_ni_b_appen} for Figure~\ref{fig:sens_to_ni}(b), and
Figure~\ref{fig:robust_to_drop_appen} for Figure~\ref{fig:robust_to_drop}.
Furthermore, Figure~\ref{fig:forecast} shows the forecasting results of \method, PatchTST, and Crossformer. We select these baselines because they have similar architecture to \method, such as segmentation or inter-feature attention modules. Our method captures temporal dynamics better than baselines.

\begin{figure}[t]
    \centering
    \includegraphics[width=0.97\textwidth]{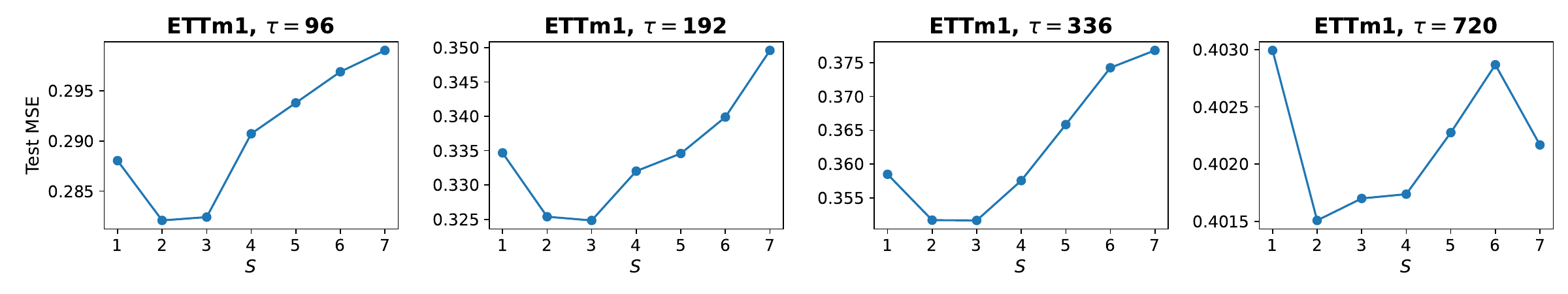}
    \includegraphics[width=0.97\textwidth]{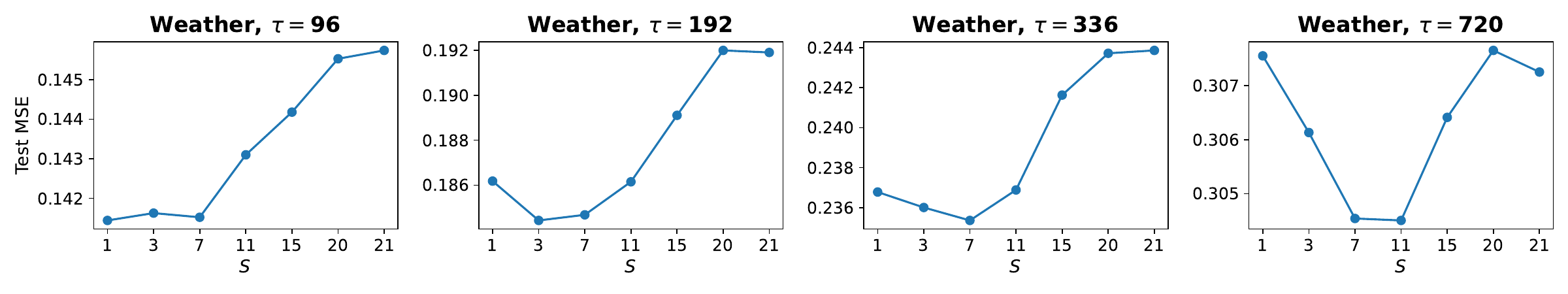}
    \includegraphics[width=0.97\textwidth]{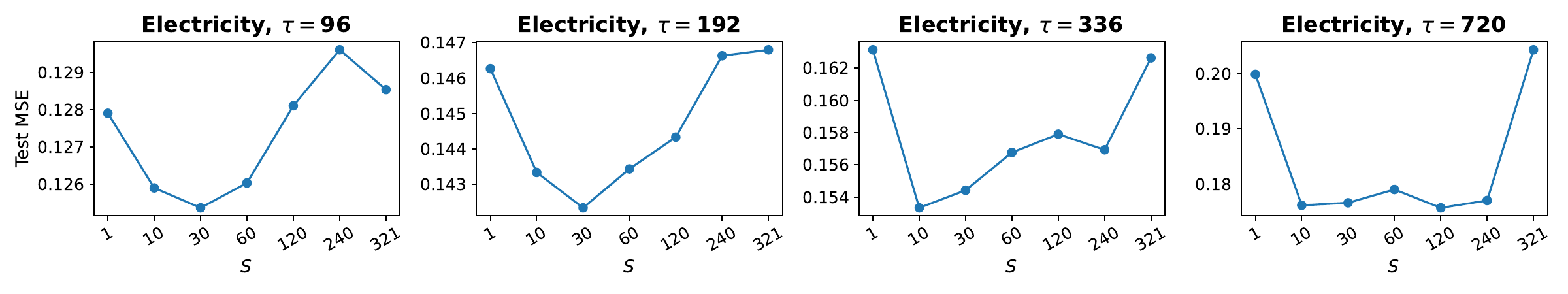}
    \includegraphics[width=0.97\textwidth]{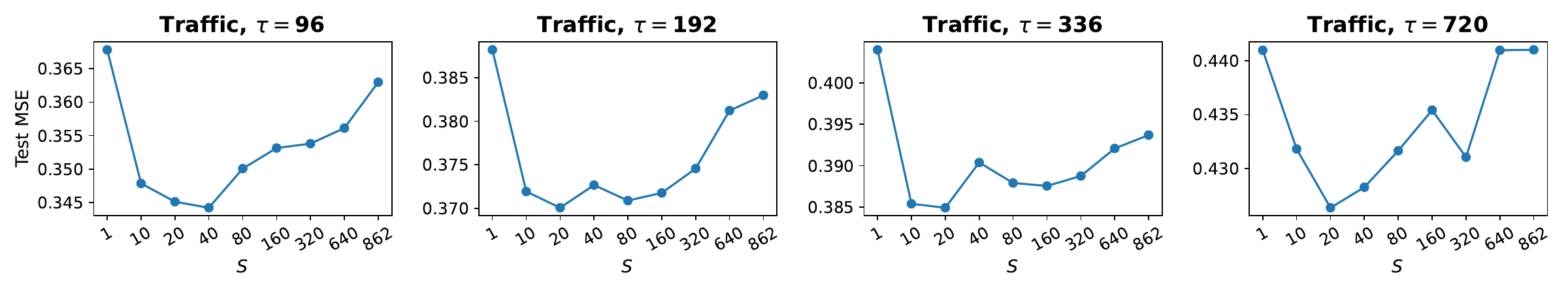}
    \vspace{-7pt}
    \caption{Sensitivity to $S$. (Additional results for Figure~\ref{fig:change_s})}\label{fig:change_s_appen}
\end{figure}

\begin{figure*}[t]
    \centering
    \includegraphics[width=0.97\textwidth]{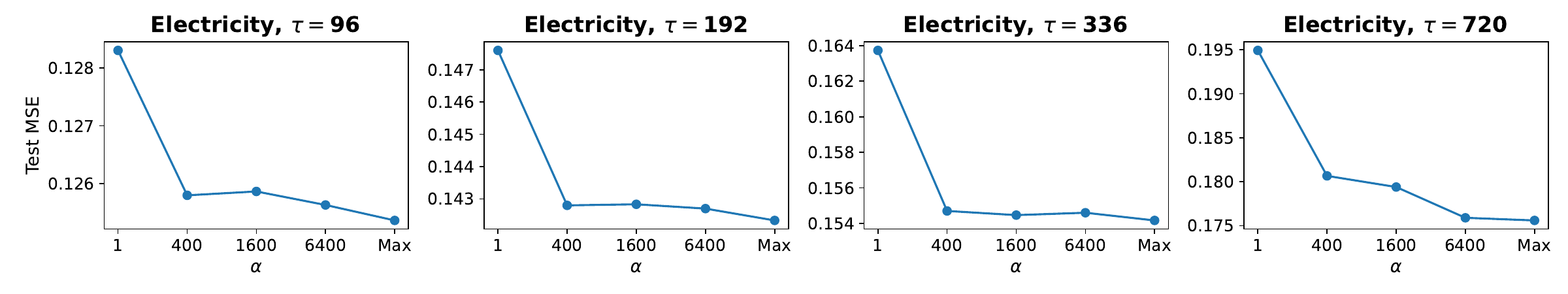}
    \includegraphics[width=0.97\textwidth]{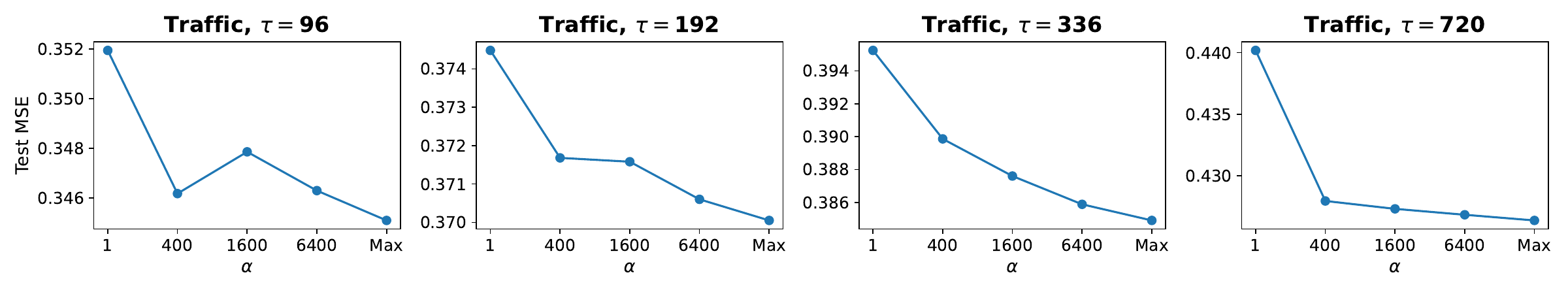}
    \vspace{-7pt}
    \caption{Sensitivity to $|\mathbf{F}^{all}|= \alpha \times N_U$. (Additional results for Figure~\ref{fig:change_f_all})}\label{fig:change_f_all_appen}
\end{figure*}
\begin{figure}[t]
    \centering
    \includegraphics[width=0.97\columnwidth]{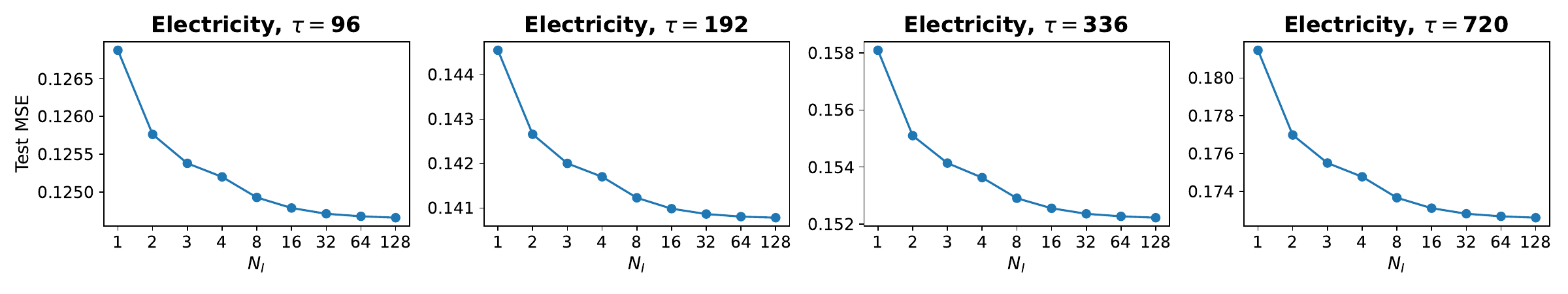}
    \includegraphics[width=0.97\columnwidth]{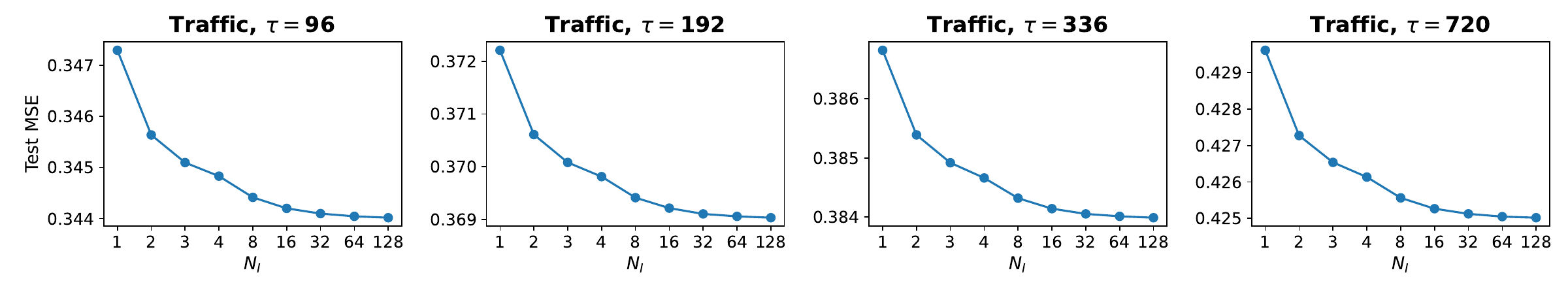}
    \vspace{-7pt}
    \caption{Sensitivity to $N_I$. (Additional results for Figure~\ref{fig:sens_to_ni}(a))}\label{fig:sens_to_ni_a_appen}
\end{figure}
\begin{figure}[t]
    \centering
    \includegraphics[width=0.97\textwidth]{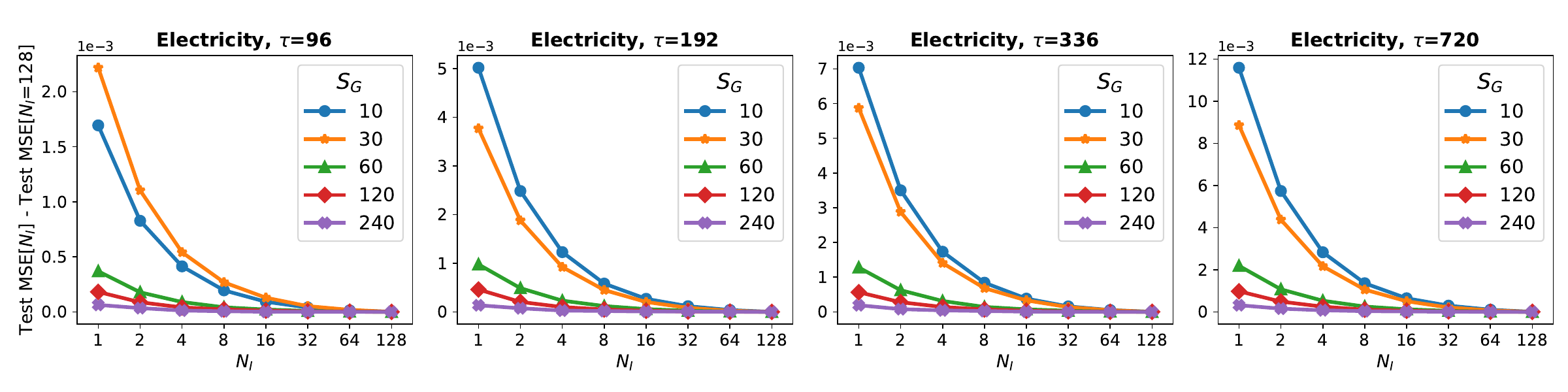}
    \includegraphics[width=0.97\textwidth]{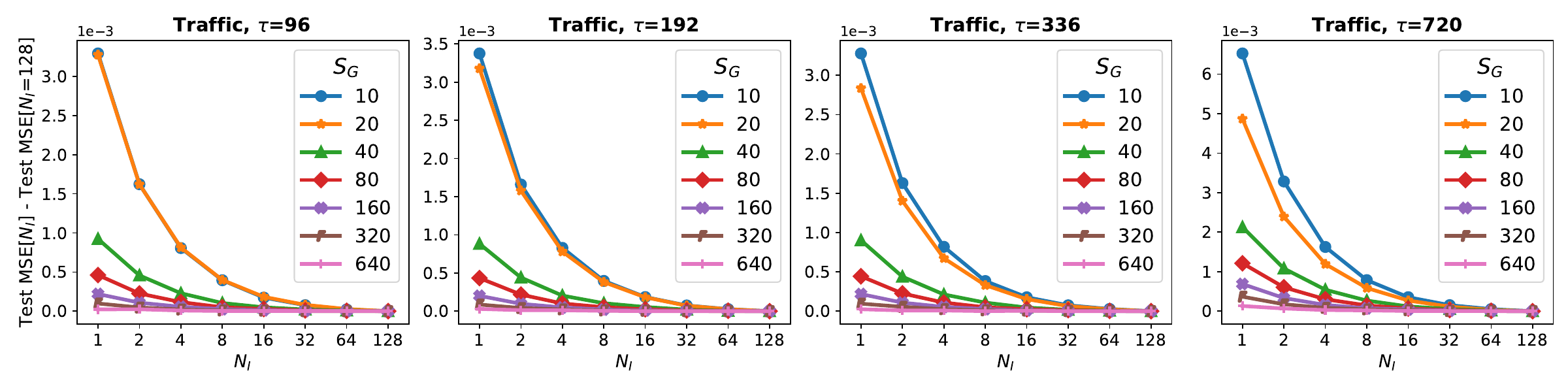}
    \vspace{-7pt}
    \caption{Changes in the effect of $N_I$ on forecasting performance when $S$ increases. (Additional results for Figure~\ref{fig:sens_to_ni}(b))}\label{fig:sens_to_ni_b_appen}
\end{figure}
\begin{figure}[t]
    \centering
    \includegraphics[width=0.97\textwidth]{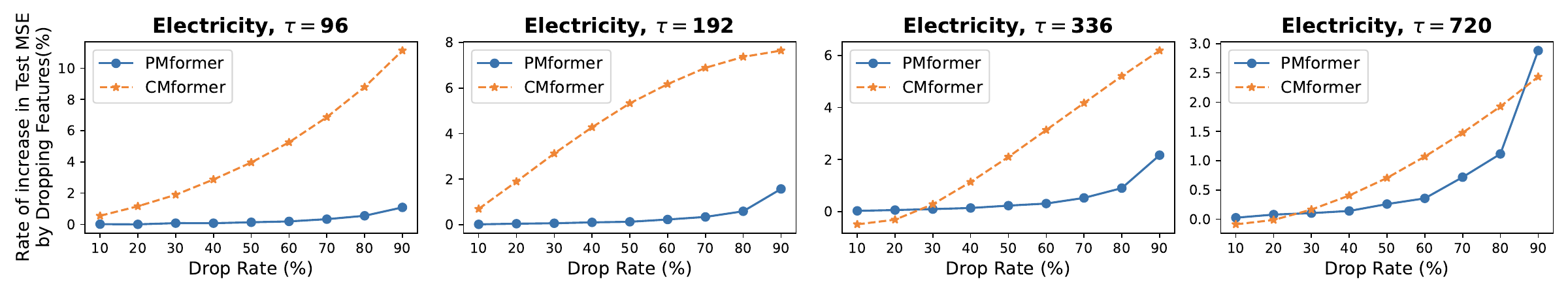}
    \includegraphics[width=0.97\textwidth]{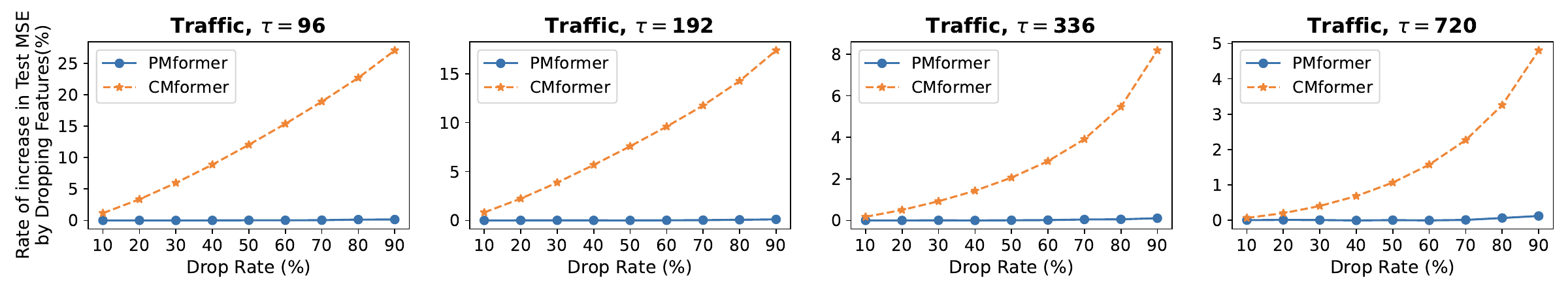}
    \vspace{-7pt}
    \caption{Increasing rate of test MSE by dropping $n$\% features in \method or Complete-Multivariate Transformer (CMformer). (Additional results for Figure~\ref{fig:robust_to_drop})}\label{fig:robust_to_drop_appen}
\end{figure}
\begin{figure*}[t]
    \centering
    \includegraphics[width=0.9\textwidth]{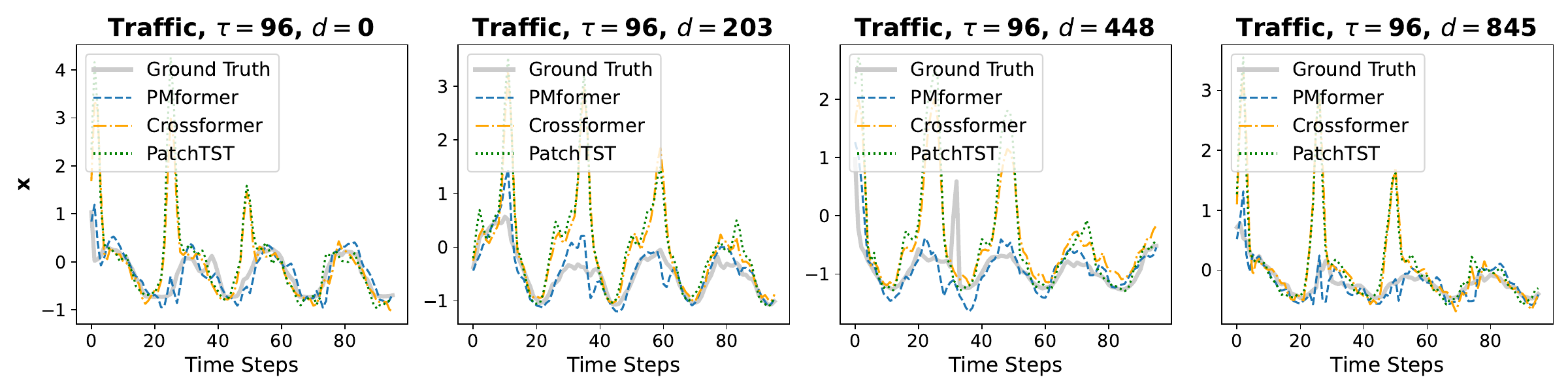}
    \includegraphics[width=0.9\textwidth]{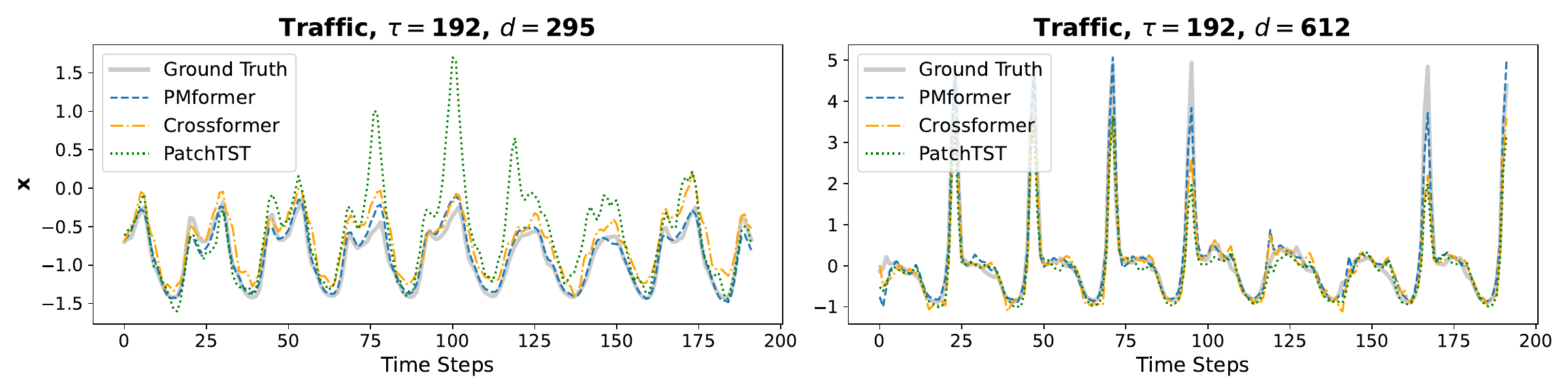}
    \includegraphics[width=0.9\textwidth]{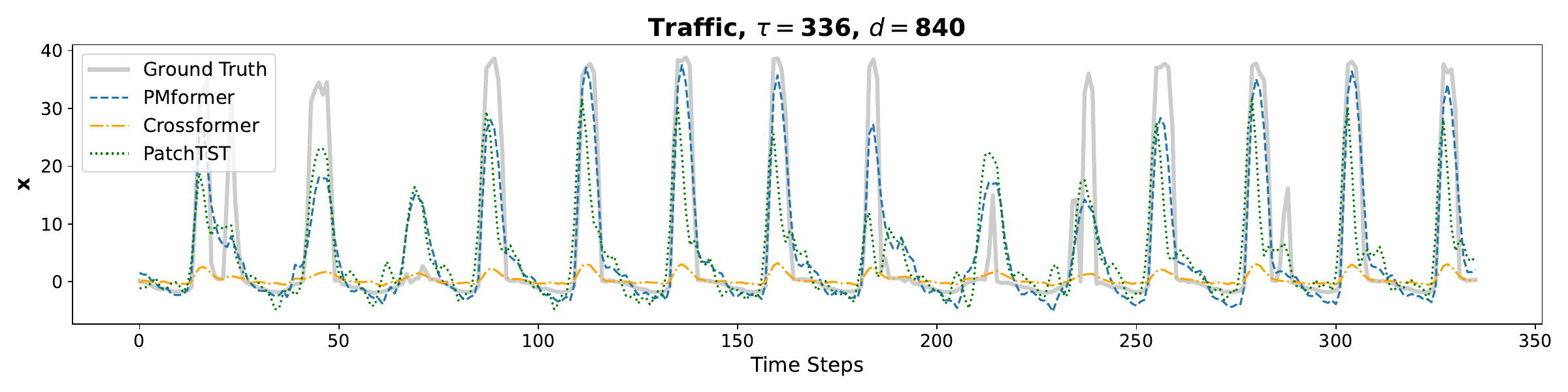}
    \includegraphics[width=0.9\textwidth]{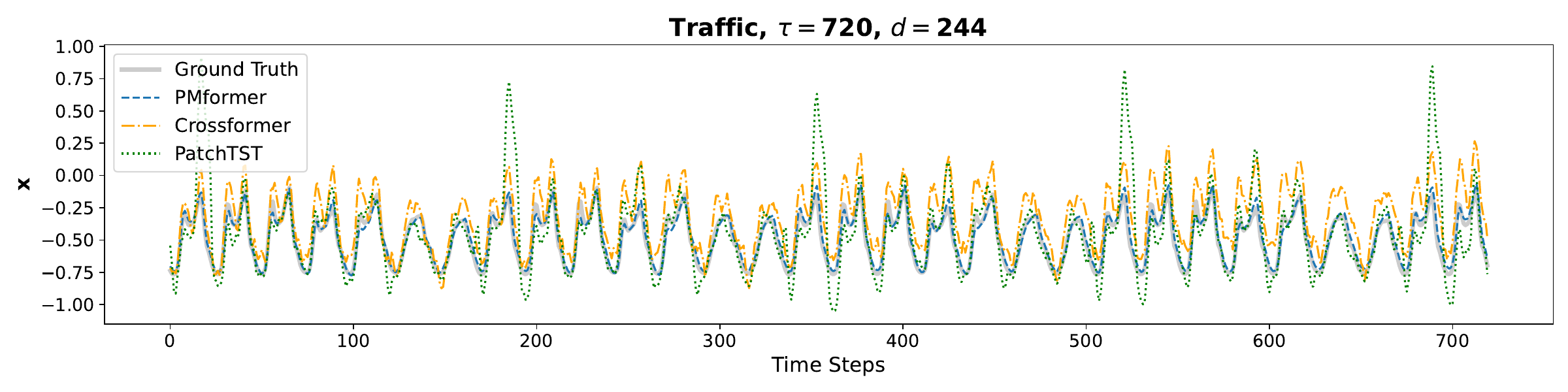}
    \vspace{-7pt}
    \caption{Forecasting results of various segment-based transformers (Crossformer, PatchTST, and \method). Dotted lines and dotted-dashed lines denote baselines, dashed lines denote \method, and solid lines denote ground truth. $\tau$ denotes the length of time steps in future outputs and $d$ denotes a feature index.}\label{fig:forecast}
\end{figure*}

\end{document}